\documentclass{article}

\usepackage{microtype}
\usepackage{graphicx}
\usepackage{subfig}
\usepackage{booktabs} %

\usepackage{enumitem}

\usepackage[pagebackref]{hyperref}

\usepackage[accepted]{icml2023}

\usepackage{amsmath}
\usepackage{amssymb}
\usepackage{mathtools}
\usepackage{amsthm}

\usepackage[capitalize,noabbrev]{cleveref}

\theoremstyle{plain}
\newtheorem{theorem}{Theorem}[section]

\newtheorem{lemma}[theorem]{Lemma}

\theoremstyle{definition}
\newtheorem{definition}[theorem]{Definition}

\theoremstyle{remark}

\usepackage[textsize=tiny]{todonotes}

\icmltitlerunning{Neural Continuous-Discrete State Space Models}

\usepackage{inconsolata}
\usepackage{graphicx}
\usepackage{mathrsfs}
\usepackage{multirow}

\usepackage{algorithmicx}
\usepackage{algpseudocode}

\usepackage{xkcdcolors}
\hypersetup{
    colorlinks=true,
    linkcolor=xkcdOcean,
    urlcolor=xkcdBluish,
    linktoc=all,
    citecolor=xkcdBrown
}%

\makeatletter
\newcommand{\LeftComment}[1]{%
  \Statex \hspace*{\ALG@thistlm}\(\triangleright\) #1}
\makeatother

\usepackage{amsmath,amsfonts,bm}

\def\eqref#1{Eq.~(\ref{#1})}
\def\eqrefp#1{(Eq.~\ref{#1})}

\def\1{\bm{1}}

\def\eps{{\epsilon}}

\def\rva{{\mathbf{a}}}

\def\rvf{{\mathbf{f}}}
\def\rvg{{\mathbf{g}}}

\def\rvm{{\mathbf{m}}}

\def\rvx{{\mathbf{x}}}
\def\rvy{{\mathbf{y}}}
\def\rvz{{\mathbf{z}}}

\def\rmA{{\mathbf{A}}}
\def\rmB{{\mathbf{B}}}
\def\rmC{{\mathbf{C}}}
\def\rmD{{\mathbf{D}}}

\def\rmF{{\mathbf{F}}}
\def\rmG{{\mathbf{G}}}
\def\rmH{{\mathbf{H}}}
\def\rmI{{\mathbf{I}}}
\def\rmJ{{\mathbf{J}}}
\def\rmK{{\mathbf{K}}}

\def\rmP{{\mathbf{P}}}
\def\rmQ{{\mathbf{Q}}}
\def\rmR{{\mathbf{R}}}
\def\rmS{{\mathbf{S}}}

\def\rmU{{\mathbf{U}}}

\def\rmX{{\mathbf{X}}}
\def\rmY{{\mathbf{Y}}}
\def\rmZ{{\mathbf{Z}}}

\DeclareMathAlphabet{\mathsfit}{\encodingdefault}{\sfdefault}{m}{sl}
\SetMathAlphabet{\mathsfit}{bold}{\encodingdefault}{\sfdefault}{bx}{n}

\def\gA{{\mathcal{A}}}

\def\gL{{\mathcal{L}}}

\def\gN{{\mathcal{N}}}

\def\gU{{\mathcal{U}}}

\def\gY{{\mathcal{Y}}}

\newcommand{\E}{\mathbb{E}}

\newcommand{\R}{\mathbb{R}}

\usepackage{thmtools,thm-restate}

\graphicspath{{figures/}}
\newcommand{\ourmodel}{NCDSSM}
\newcommand{\ctssmlti}{NCDSSM-LTI}
\newcommand{\ctssmll}{NCDSSM-LL}
\newcommand{\ctssmnl}{NCDSSM-NL}

\begin{document}

\twocolumn[
\icmltitle{Neural Continuous-Discrete State Space Models \\ for Irregularly-Sampled Time Series}

\icmlsetsymbol{note}{$\dagger$}

\begin{icmlauthorlist}
\icmlauthor{Abdul Fatir Ansari}{amazon,note}
\icmlauthor{Alvin Heng}{nus}
\icmlauthor{Andre Lim}{nus}
\icmlauthor{Harold Soh}{nus,ssi}
\end{icmlauthorlist}

\icmlaffiliation{nus}{School of Computing, National University of Singapore (NUS)}
\icmlaffiliation{ssi}{Smart Systems Institute, NUS}
\icmlaffiliation{amazon}{AWS AI Labs}

\icmlcorrespondingauthor{Abdul Fatir Ansari}{abdulfatir@u.nus.edu}

\icmlkeywords{Machine Learning, ICML}

\vskip 0.3in
]

\printAffiliationsAndNotice{\textsuperscript{$\dagger$}Work done while at National University of Singapore, prior to joining Amazon.}  %

\begin{abstract}
Learning accurate predictive models of real-world dynamic phenomena (e.g., climate, biological) remains a challenging task. One key issue is that the data generated by both natural and artificial processes often comprise time series that are irregularly sampled and/or contain missing observations. In this work, we propose the Neural Continuous-Discrete State Space Model (\ourmodel) for continuous-time modeling of time series through discrete-time observations. \ourmodel\ employs auxiliary variables to disentangle recognition from dynamics, thus requiring amortized inference only for the auxiliary variables. Leveraging techniques from continuous-discrete filtering theory, we demonstrate how to perform accurate Bayesian inference for the dynamic states. We propose three flexible parameterizations of the latent dynamics and an efficient training objective that marginalizes the dynamic states during inference. Empirical results on multiple benchmark datasets across various domains show improved imputation and forecasting performance of \ourmodel\ over existing models.

\end{abstract}

\vspace{-2.5em}
\section{Introduction}
\vspace{-.5em}
State space models (SSMs) provide an elegant framework for modeling time series data. Combinations of SSMs with neural networks have proven effective for various time series tasks such as segmentation, imputation, and forecasting~\citep{krishnan2015deep,fraccaro2017disentangled,rangapuram2018deep,kurle2020deep,ansari2021deep}. However, most existing models are limited to the discrete time (i.e., uniformly sampled) setting, whereas data from various physical~\citep{menne2010long}, biological~\citep{PhysioNet}, and business~\citep{turkmen2019intermittent} systems in the real world are sometimes only available at irregular intervals. Such systems are best modeled as continuous-time latent processes with irregularly-sampled discrete-time observations. Desirable features of such a time series model include modeling of stochasticity (uncertainty) in the system, and efficient and accurate inference of the system state from potentially high-dimensional observations (e.g., video frames). 

\begin{figure}
    \centering
    \includegraphics[width=0.80\linewidth]{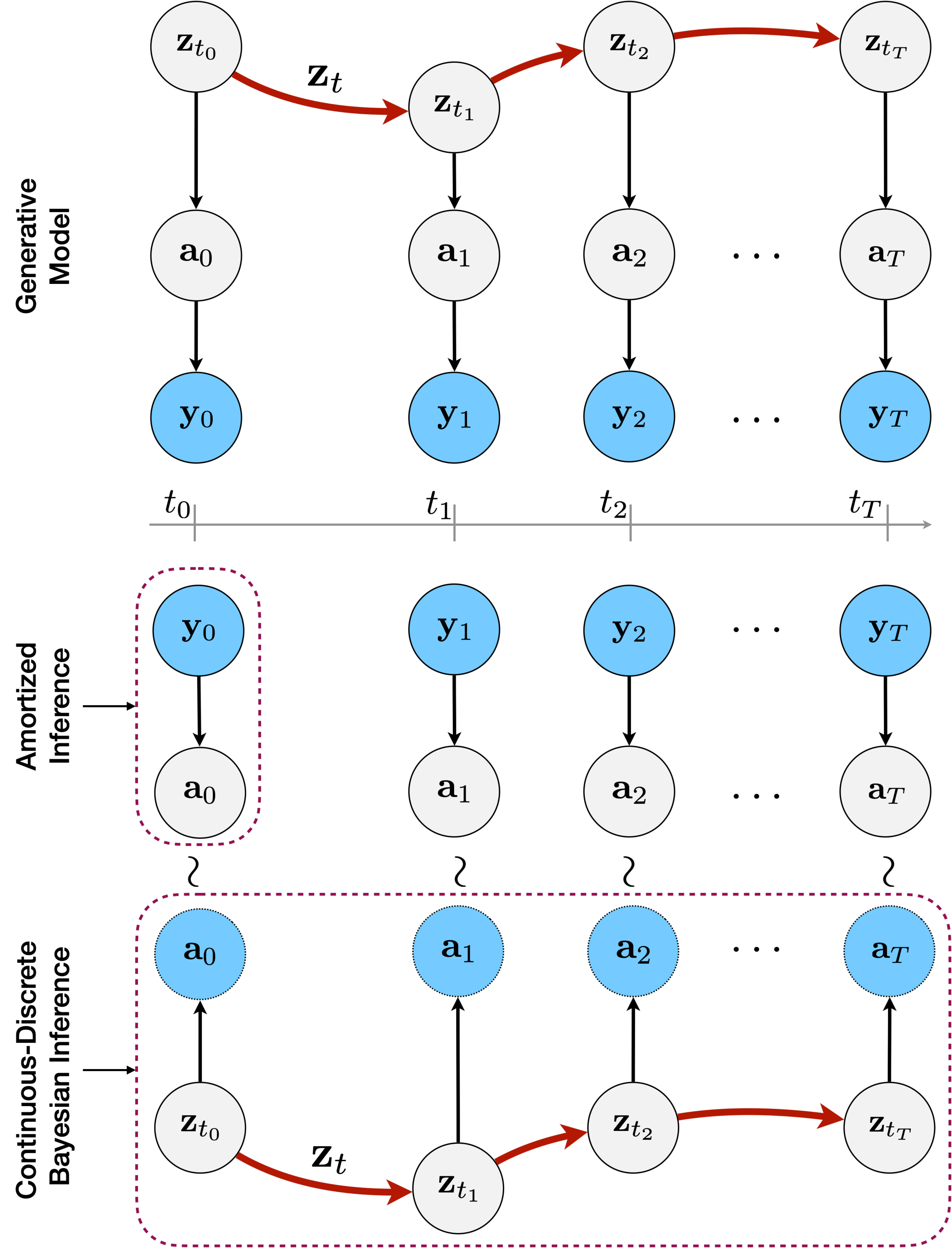}
    \caption{(\textbf{Top}) Generative model of Neural Continuous-Discrete State Space Model. The bold red arrows indicate that the state, $\rvz_t$, evolves continuously in time. The auxiliary variables, $\rva_k$, and observations, $\rvy_k$, are emitted at arbitrary discrete timesteps $t_k \in \{t_0, t_1, \dots, t_T\}$. (\textbf{Bottom}) Amortized inference for auxiliary variables and continuous-discrete Bayesian inference for states. Samples from the amortized variational distribution over auxiliary variables are used as pseudo-observations to condition and perform inference in the continuous-discrete SSM at the bottom.}
    \label{fig:gen-model-and-inference}
    \vspace{-2em}
\end{figure}

Recently, latent variable models based on neural differential equations have gained popularity for continuous-time modeling of time series~\citep{chen2018neural,rubanova2019latent,yildiz2019ode2vae,li2020scalable,liu2020learning,solin2021scalable}. However, these models suffer from limitations. The ordinary differential equation (ODE)-based models employ deterministic latent dynamics and/or encode the entire context window into an initial state, creating a restrictive bottleneck. 
Stochastic differential equation (SDE)-based models use stochastic latent dynamics, but typically perform a variational approximation of the latent trajectories via posterior SDEs. The posterior SDEs incorporate new observations in an ad-hoc manner, potentially resulting in a disparity between the posterior and generative transition dynamics, and a non-Markovian state space.

To address these issues, we propose the Neural Continuous-Discrete State Space Model (\ourmodel) that uses discrete-time observations to model continuous-time stochastic Markovian dynamics (Fig.~\ref{fig:gen-model-and-inference}). By using auxiliary variables, \ourmodel{} disentangles recognition of high-dimensional observations from dynamics (encoded by the state)~\citep{fraccaro2017disentangled,kurle2020deep}. We leverage the rich literature on continuous-discrete filtering theory~\citep{jazwinski1970stochastic}, which has remained relatively underexplored in the modern deep learning context~\citep{schirmer2022modeling}. Our proposed inference algorithm only performs amortized variational inference for the auxiliary variables since they enable classic continuous-discrete Bayesian inference~\cite{jazwinski1970stochastic} for the states, using only the generative model. This obviates the need for posterior SDEs and allows incorporation of new observations via a principled Bayesian update, resulting in accurate state estimation. As a result, \ourmodel\ enables online prediction and naturally provides state uncertainty estimates. We propose three dynamics parameterizations for \ourmodel\ (linear time-invariant, non-linear and locally-linear) and a training objective that can be easily computed during inference.

We evaluated \ourmodel\ on imputation and forecasting tasks on multiple benchmark datasets. Our experiments demonstrate that \ourmodel\ accurately captures the underlying dynamics of the time series and extrapolates it consistently beyond the training context, significantly outperforming baseline models. From a practical perspective, we found that \ourmodel\ is less sensitive to random initializations and requires fewer parameters than the baselines.

In summary, the key contributions of this work are:
\begin{itemize}[noitemsep,nolistsep,topsep=-0.5em]
    \item \ourmodel, a continuous-discrete SSM with auxiliary variables for continuous-time modeling of irregularly-sampled (high dimensional) time series;
    \item An accurate inference algorithm that performs amortized inference for auxiliary variables and classic Bayesian inference for the dynamic states;
    \item An efficient learning algorithm and its stable implementation using square root factors;
    \item Experiments on multiple benchmark datasets, demonstrating that \ourmodel\ learns accurate models of the underlying dynamics and extrapolates it consistently into the future.
\end{itemize}

\vspace{-0.5em}
\section{Approximate Continuous-Discrete Inference}
\label{sec:background}
\vspace{-0.5em}
We begin with a review of approximate continuous-discrete Bayesian filtering and smoothing, inference techniques employed by our proposed model. Consider the following It\^o SDE,
\begin{equation}
    d\rvz_t = \rvf(\rvz_t, t) dt + \rmG(\rvz_t, t)d\rmB_t,\label{eq:non-linear-sde}
\end{equation}
where $\rvz_t \in \R^m$ is the state, $\rmB_t \in \R^m$ denotes a Brownian motion with diffusion matrix $\rmQ$, $\rvf(\cdot, t): \R^m \to \R^m$ is the drift function and $\rmG(\cdot, t): \R^m \to \R^{m \times m}$ is the diffusion function at time $t$. The initial density of the state, $p(\rvz_0)$, is assumed to be known and independent of the Brownian motion, $\rmB_t$. The evolution of the marginal density of the state, $p_t(\rvz_t)$, is governed by the Fokker-Plank-Kolmogorov (FPK) equation~\citep[Ch.~4]{jazwinski1970stochastic},
\begin{align}
    \frac{\partial p_t(\rvz_t)}{\partial t} = \mathscr{L}^\ast p_t,\label{eq:fpk}
\end{align}
where $\mathscr{L}^\ast$ is the forward diffusion operator given by
\begin{equation}
    \mathscr{L}^\ast \varphi = -\sum_{i=1}^d\frac{\partial}{\partial x_i}\left[\varphi f_i\right] 
    + \frac{1}{2}\sum_{i=1}^d\sum_{j=1}^d\left[\varphi(\rmG\rmQ \rmG^\top)_{ij}\right].\nonumber
\end{equation}
In practice, we only have access to noisy transformations (called measurements or observations), $\rva_{t_k} \in \R^d$, of the state, $\rvz_{t_k}$, at discrete timesteps $t_k \in \{t_0, \dots, t_T\}$. In the following, we employ the notation $\rvx_t$ to represent the value of a variable $\rvx$ at an arbitrary continuous time $t$, and $\rvx_{t_k}$ (or $\rvx_k$ for short) to represent its value at the time $t_k$ associated with the $k$-th discrete timestep. The \emph{continuous-discrete state space model}~\citep[Ch.~6]{jazwinski1970stochastic} is an elegant framework for modeling such time series. 

\begin{definition}[Continuous-Discrete State Space Model]
    A continuous-discrete state space model is one where the latent state, $\rvz_t$, follows the continuous-time dynamics governed by \eqref{eq:non-linear-sde} and the measurement, $\rva_k$, at time $t_k$ is obtained from the measurement model $p(\rva_k | \rvz_{t_k})$. 
\end{definition}

In this work, we consider linear Gaussian measurement models,
$
    \rva_k \sim \gN(\rva_k; \rmH\rvz_{t_k}, \rmR),
$
where $\rmH \in \R^{d \times m}$ is the measurement matrix and $\rmR \succeq 0 \in \R^{d \times d}$ is the measurement covariance matrix.
Given observations $\gA_\tau = \{\rva_k: t_k \leq \tau\}$, we are interested in answering two types of inference queries: the posterior distribution of the state, $\rvz_t$, conditioned on observations up to time $t$, $p_t(\rvz_t | \gA_t)$, and the posterior distribution of the state, $\rvz_t$, conditioned on all available observations, $p_t(\rvz_t | \gA_T)$. These are known as the \emph{filtering} and \emph{smoothing} problems, respectively. 

The filtering density, $p_t(\rvz_t | \gA_t)$, satisfies the FPK equation (Eq.~\ref{eq:fpk}) for $t \in [t_k, t_{k+1})$ between observations, with the initial condition $p_t(\rvz_t | \gA_{t_k})$ at time $t_k$. Observations can be incorporated via a Bayesian update,
\begin{equation}
    p_t(\rvz_{t_k} | \gA_{t_k}) = \frac{p(\rva_{t_k} | \rvz_{t_k})p(\rvz_{t_k} | \gA_{t_{k-1}})}{p(\rva_k | \gA_{t_{k-1}})}.\label{eq:bayesian-update}
\end{equation}
The smoothing density satisfies a backward partial differential equation related to the FPK equation. We refer the reader to \citet{anderson1972fixed} and \citet[Ch.~10]{sarkka2019applied} for details and discuss a practical approximate filtering procedure in the following (cf. Appendix~\ref{app:smoothing} for smoothing).
\subsection{Continuous-Discrete Bayesian Filtering} 
Solving \eqref{eq:fpk} for arbitrary $\rvf$ and $\rmG$ is intractable; hence, several approximations have been considered in the literature~\citep[Ch.~9]{sarkka2019applied}. The Gaussian assumed density approximation uses a Gaussian approximation,
\begin{equation}
    p_t(\rvz_t) \approx \gN(\rvz_t; \rvm_t, \rmP_t),\label{eq:assumed-density-approx}
\end{equation}
for the solution to the FPK equation, characterized by the time-varying mean, $\rvm_t$, and covariance matrix, $\rmP_t$. Further, linearization of the drift $\rvf$ via Taylor expansion results in the following ODEs that govern the evolution of the mean and covariance matrix,
\begin{subequations}
\label{eq:linearization-approx}
\begin{align}
    \frac{d\rvm_t}{dt} &= \rvf(\rvm_t, t),\label{eq:mean-predict}\\
    \frac{d\mathbf{P}_t}{dt} &= \rmF_{\rvz}(\rvm_t, t)\rmP_t + \rmP_t\rmF^\top_{\rvz}(\rvm_t, t) + \rmD(\rvm_t, t),\label{eq:cov-predict}
\end{align}
\end{subequations}
where $\rmF_{\rvz}(\rvm_t, t)$ is the Jacobian of $\rvf(\rvz, t)$ with respect to $\rvz$ at $\rvm_t$ and $\rmD(\cdot, t) = \rmG(\cdot, t)\rmQ \rmG^\top(\cdot, t)$. Thus, for $t \in [t_k, t_{k+1})$ between observations, the filter distribution $p_t(\rvz_t | \gA_t)$ can be approximated as a Gaussian with mean and covariance matrix given by solving \eqref{eq:linearization-approx}, with initial conditions $\rvm_{t_k}$ and $\rmP_{t_k}$ at time $t_k$. This is known as the \emph{prediction step}.

The Gaussian assumed density approximation of $p(\rvz_{t_k} | \gA_{t_{k-1}})$ described above makes the Bayesian update in \eqref{eq:bayesian-update} analytically tractable as $p(\rva_{t_k} | \rvz_{t_k})$ is also a Gaussian distribution with mean $\rmH\rvz_{k}$ and covariance matrix $\rmR$. The parameters, $\rvm_{k}$ and $\mathbf{P}_{k}$, of the Gaussian approximation of $p_t(\rvz_{t_k} | \gA_{t_k})$ are then given by,
\begin{subequations}
\label{eq:update-step}
\begin{align}
    \rmS_k &= \mathbf{H}\mathbf{P}_{k}^-\mathbf{H}^\top + \mathbf{R},\\
    \mathbf{K}_k &= \mathbf{P}_{k}^-\mathbf{H}^\top\rmS_k^{-1},\label{eq:kalman-gain}\\
    \rvm_{k} &= \rvm_{k}^- + \mathbf{K}_k\left(\rva_k - \mathbf{H}\rvm_{k}^-\right),\\
    \mathbf{P}_{k} &= \mathbf{P}_{k}^- - \mathbf{K}_k\rmS_k\mathbf{K}^\top_k,
\end{align}
\end{subequations}
where $\rvm_{k}^-$ and $\mathbf{P}_{k}^-$ are the parameters of $p_t(\rvz_{t_k} | \gA_{t_{k-1}})$ given by the prediction step. \eqref{eq:update-step} constitutes the \emph{update step} which is exactly the same as the update step in the Kalman filter for discrete-time linear Gaussian SSMs. The continuous-time prediction step together with the discrete-time update step is sometimes also referred to as the hybrid Kalman filter.
As a byproduct, the update step also provides the conditional likelihood terms,
$
    p(\rva_k | \gA_{t_{k-1}}) = \gN(\rva_k; \mathbf{H}\rvm_{k}^-, \rmS_k),
$
which can be combined to give the likelihood of the observed sequence,
$
    p(\gA_{t_{T}}) = p(\rvy_0)\prod_{k=1}^T p(\rva_k | \gA_{t_{k-1}}).
$
\vspace{-0.5em}
\section{Neural Continuous-Discrete State Space Models}
\vspace{-0.5em}
In this section, we describe our proposed model: Neural Continuous-Discrete State Space Model (\ourmodel). We begin by formulating \ourmodel\ as a continuous-discrete SSM with auxiliary variables that serve as succinct representations of high-dimensional observations. %
We then discuss how to perform efficient inference along with parameter learning and a stable implementation for \ourmodel.

\subsection{Model Formulation}
\label{sec:model-formulation}

\ourmodel\ is a continuous-discrete SSM in which the latent state, $\rvz_t \in \R^m$, evolves in continuous time, emitting linear-Gaussian auxiliary variables, $\rva_t \in \R^h$, which in turn emit observations, $\rvy_t \in \R^d$. Thus, \ourmodel\ possesses two types of latent variables: (a) the states that encode the hidden dynamics, and (b) the auxiliary variables that can be viewed as succinct representations of the observations and are equivalent to observations in the continuous-discrete state space models considered in Section~\ref{sec:background}. The inclusion of auxiliary variables offers two benefits; (i) it allows disentangling representation learning (or recognition) from dynamics (encoded by $\rvz_t$) and (ii) it enables the use of arbitrary decoders to model the conditional distribution $p(\rvy_t | \rva_t)$. We discuss this further in Section \ref{sec:inference}.

Consider the case when we have observations available at discrete timesteps $t_0, \dots, t_T$. Following the graphical model in Fig.~\ref{fig:gen-model-and-inference}, the joint distribution over the states $\rvz_{0:T}$, the auxiliary variables $\rva_{0:T}$, and the observations $\rvy_{0:T}$ factorises as
\begin{align*}
 p_{\theta}(\rvz_{0:T}, &\rva_{0:T}, \rvy_{0:T}) = \prod_{k=0}^T p(\rvy_{k} | \rva_{k})p(\rva_{k} | \rvz_{k})p(\rvz_{k} | \rvz_{{k-1}}),
\end{align*}
where $\rvx_{0:T}$ denotes the set $\{\rvx_{t_0}, \dots, \rvx_{t_T}\}$ and $p(\rvz_{0} | \rvz_{{-1}}) = p(\rvz_{0})$. We model the initial (prior) distribution of the states as a multivariate Gaussian distribution,
\begin{equation}
    p(\rvz_{0}) = \gN(\rvz_0; \bm{\mu}_0, \bm{\Sigma}_0),
\end{equation}
where $\bm{\mu}_0 \in \R^m$ and $\bm{\Sigma}_0 \succeq 0 \in \R^{m \times m}$ are the mean and covariance matrix, respectively. The transition distribution of the states, $p(\rvz_{k} | \rvz_{{k-1}})$, follows the dynamics governed by the SDE in \eqref{eq:non-linear-sde}. The conditional emission distributions of the auxiliary variables and observations are modeled as multivariate Gaussian distributions given by,
\begin{align}
    p(\rva_{k} | \rvz_{k}) &= \gN(\rva_{k}; \mathbf{H}\rvz_{k}, \rmR),\label{eq:auxiliary-emission}\\
    p(\rvy_{k} | \rva_{k}) &= \gN(\rvy_{k}; f^\mu(\rva_{k}), f^\Sigma(\rva_{k})),\label{eq:observation-emission}
\end{align}
where $\rmH~\in~\R^{h \times m}$ is the auxiliary measurement matrix, $\rmR~\succeq~0~\in~\R^{h \times h}$ is the auxiliary covariance matrix, and $f^\mu$ and $f^\Sigma$ are functions parameterized by neural networks that output the mean and the covariance matrix of the distribution, respectively. We use $\theta$ to denote the parameters of the generative model, including SSM parameters $\{\bm{\mu}_0, \bm{\Sigma}_0, \rvf, \rmQ, \rmG, \rmH, \rmR\}$ and observation emission distribution parameters $\{f^\mu, f^\Sigma\}$.

We propose three variants of \ourmodel, depending on the parameterization of $\rvf$ and $\rmG$ functions in \eqref{eq:non-linear-sde} that govern the dynamics of the state:

\textbf{Linear time-invariant dynamics}~is obtained by parameterizing $\rvf$ and $\rmG$ as
\begin{equation}
    \rvf(\rvz_t, t) = \rmF\rvz_t \quad \text{and} \quad \rmG(\rvz, t) = \rmI,\label{eq:lti-dynamics}
\end{equation}
respectively, where $\rmF \in \R^{m \times m}$ is a Markov transition matrix and $\rmI$ is the $m$-dimensional identity matrix. In this case, Eqs.\ (\ref{eq:assumed-density-approx}) and (\ref{eq:linearization-approx}) become exact and the ODEs in \eqref{eq:linearization-approx} can be solved analytically using matrix exponentials (cf.~Appendix~\ref{app:algorithms}). Unfortunately, the restriction of linear dynamics is limiting for practical applications. We denote this linear time-invariant variant as \ctssmlti.

\textbf{Non-linear dynamics}~is obtained by parameterizing $\rvf$ and $\rmG$ using neural networks. With sufficiently powerful neural networks, this parameterization is flexible enough to model arbitrary non-linear dynamics. However, the neural networks need to be carefully regularized (cf. Appendix~\ref{app:stable-implementation}) to ensure optimization and inference stability. Inference in this variant also requires computation of the Jacobian of a neural network for solving \eqref{eq:linearization-approx}. We denote this non-linear variant as \ctssmnl.

\textbf{Locally-linear dynamics}~is obtained by parameterizing $\rvf$ and $\rmG$ as
\begin{equation}
    \rvf(\rvz_t, t) = \rmF(\rvz_t)\rvz_t \quad \text{and} \quad \rmG(\rvz, t) = \rmI,\label{eq:locally-linear-dynamics}
\end{equation}
respectively, where the matrix $\rmF(\rvz_t) \in \R^{m \times m}$ is given by a convex combination of $K$ base matrices $\{\rmF^{(j)}\}_{j=1}^K$,
\begin{equation}
    \rmF(\rvz_t) = \sum_{j=1}^K \alpha^{(j)}(\rvz_t)\rmF^{(j)},\label{eq:locally-linear}
\end{equation}
and the combination weights, $\alpha(\rvz_t)$, are given by
\begin{equation}
    \alpha(\rvz_t) = \mathrm{softmax}(g(\rvz_t)),\label{eq:locally-linear-mixture-weights}
\end{equation}
where $g$ is a neural network. Such parameterizations smoothly interpolate between linear SSMs and can be viewed as ``soft'' switching SSMs. Locally-linear dynamics has previously been used for discrete-time SSMs~\citep{karl2016deep,klushyn2021latent}; we extend it to the continuous time setting by evaluating \eqref{eq:locally-linear} continuously in time. Unlike non-linear dynamics, this parameterization does not require careful regularization and its flexibility can be controlled by choosing the number of base matrices, $K$. Furthermore, the Jacobian of $\rvf$ in \eqref{eq:linearization-approx} can be approximated as $\rmF(\rvm_t)$, avoiding the expensive computation of the Jacobian of a neural network~\citep{klushyn2021latent}. We denote this locally-linear variant as \ctssmll.

\subsection{Inference}
\label{sec:inference}
Exact inference in the model described above is intractable when the dynamics is non-linear and/or the observation emission distribution, $p(\rvy_{k} | \rva_{k})$, is modeled by arbitrary non-linear functions. In the modern deep learning context, a straightforward approach would be to approximate the posterior distribution over the states and auxiliary variables, $q(\rvz_{0:T}, \rva_{0:T} | \rvy_{0:T})$, using recurrent neural networks (e.g., using ODE-RNNs when modeling in continuous time). However, such parameterizations have been shown to lead to poor optimization of the transition model in discrete-time SSMs, leading to inaccurate learning of system dynamics~\citep{klushyn2021latent}. Alternatively, directly applying continuous-discrete inference techniques to non-linear emission models requires computation of Jacobian matrices and inverses of $d \times d$ matrices (cf. Eq.~\ref{eq:update-step}) which scales poorly with the data dimensionality. 

The introduction of linear-Gaussian auxiliary variables offers a middle ground between the two options above. It allows efficient use of continuous-discrete Bayesian inference techniques for the inference of states, avoiding fully amortized inference for auxiliary variables and states. Concretely, we split our inference procedure into two inference steps: (i) for auxiliary variables and (ii) for states.
\paragraph{Inference for auxiliary variables.} We perform amortized inference for the auxiliary variables, factorizing the variational distribution as,
\begin{equation}
    q_\phi(\rva_{0:T} | \rvy_{0:T}) = \prod_{k=0}^T q(\rva_{k} | \rvy_{k}),\label{eq:auxiliary-inference}
\end{equation}
where $q(\rva_{k} | \rvy_{k}) = \gN(\rva_{k}; f_{\phi}^\mu(\rvy_{k}), f_{\phi}^\Sigma(\rvy_{k}))$ and $f_{\phi}^\mu$, $f_{\phi}^\Sigma$ are neural networks. This can be viewed as the recognition network in a variational autoencoder, per timestep. This flexible factorization permits use of arbitrary recognition networks, thereby allowing arbitrary non-linear emission distributions, $p(\rvy_{k} | \rva_{k})$. 

\paragraph{Inference for states.} Given the variational distribution $q_\phi(\rva_{0:T} | \rvy_{0:T})$ in \eqref{eq:auxiliary-inference}, we can draw samples, $\tilde{\rva}_{0:T} \sim q_\phi(\rva_{0:T} | \rvy_{0:T})$, from it. Viewing $\tilde{\rva}_{0:T}$ as pseudo-observations, we treat the remaining SSM (i.e., the states and auxiliary variables) separately. Specifically, conditioned on the auxiliary variables, $\tilde{\gA}_\tau = \{\tilde{\rva}_k: t_k \leq \tau\}$, we can answer inference queries over the states $\rvz_t$ in continuous time. This does not require additional inference networks and can be performed only using the generative model via classic continuous-discrete Bayesian inference techniques in  Section~\ref{sec:background}. To infer the filtered density, $p_t(\rvz_t | \tilde{\gA}_t)$, we can use \eqref{eq:linearization-approx} for the prediction step and \eqref{eq:update-step} for the update step, replacing $\rvy_k$ by $\tilde{\rva}_k$. Similarly, we can use \eqref{eq:smooth-linearization-approx} (Appendix) to infer the smoothed density, $p_t(\rvz_t | \tilde{\gA}_T)$.

As the inference of states is now conditioned on auxiliary variables, only the inversion of $h \times h$ matrices is required which is computationally feasible as $\rva_k$ generally has lower dimensionality than $\rvy_k$. Notably, this inference scheme does not require posterior SDEs for inference (as in other SDE-based models; cf. Section~\ref{sec:related-work})
and does not suffer from poor optimization of the transition model as we employ the (generative) transition model for the inference of states.

\subsection{Learning}
\label{sec:learning}

The parameters of the generative model $\{\theta\}$ and the inference network $\{\phi\}$ can be jointly optimized by maximizing the following evidence lower bound (ELBO) of the log-likelihood, $\log p_{\theta}(\rvy_{0:T})$,
\begin{align}
    &\log p_{\theta}(\rvy_{0:T})\nonumber\\
    &\:\geq \E_{q_\phi(\rva_{0:T} | \rvy_{0:T})}\left[\log \frac{\prod_{k=0}^T p_\theta(\rvy_{k} | \rva_{k})p_\theta(\rva_{0:T})}{\prod_{k=0}^T q_\phi(\rva_{k} | \rvy_{k})}\right]\nonumber\\
    &\: =: \gL_{\mathrm{ELBO}}(\theta,\phi).
\end{align}

The distributions $p_\theta(\rvy_{k} | \rva_{k})$ and $q_\phi(\rva_{k} | \rvy_{k})$ in $\gL_{\mathrm{ELBO}}$ are immediately available via the emission and recognition networks, respectively. What remains is the computation of $p_\theta(\rva_{0:T})$. Fortunately, $p_\theta(\rva_{0:T})$ can be computed as a byproduct of the inference (filtering) procedure described in Section~\ref{sec:inference}. The distribution factorizes as
$$
    p(\rva_{0:T}) = p(\rva_0)\prod_{k=1}^T p(\rva_k | \gA_{t_{k-1}}),
$$
where $p(\rva_k | \gA_{t_{k-1}}) = \gN(\rva_k; \mathbf{H}\rvm_{k}^-, \rmS_k)$, and $\rvm_{k}^-$ and $\rmS_k$ are computed during the prediction and update steps, respectively. The $p_\theta(\rva_{0:T})$ term can be viewed as a ``prior'' over the auxiliary variables. However, unlike the fixed standard Gaussian prior in a vanilla variational autoencoder, $p_\theta(\rva_{0:T})$ is a learned prior given by the marginalization of the states, $\rvz_t$, from the underlying SSM. Algorithm \ref{alg:learning} summarizes the learning algorithm for a single time series; in practice, mini-batches of time series are sampled from the dataset.

\begin{algorithm}
\caption{Learning in Neural Continuous-Discrete State Space Models}\label{alg:learning}
\begin{algorithmic}[1]
\Require{Observations $\{(\rvy_k, t_k)\}_{k=0}^{T}$ and model parameters $\{\theta, \phi\}$.}
\Repeat
    \State{Compute $q_\phi(\rva_{0:T} | \rvy_{0:T})$ using \eqref{eq:auxiliary-inference}.}
    \State{Sample $\tilde{\rva}_{0:T} \sim q_\phi(\rva_{0:T} | \rvy_{0:T})$.}
    \State{$\underline{\hspace{1em}}, \log p_\theta(\tilde{\rva}_{0:T}) \gets$  \Call{Filter}{$\tilde{\rva}_{0:T}$, $t_{0:T}$; $\theta$}}
    \LeftComment{cf. Algorithm~\ref{alg:filtering} (Appendix) for \textsc{Filter}.}
    \State{Compute $\prod_{k=0}^T p_\theta(\rvy_k|\tilde{\rva}_k)$ using \eqref{eq:observation-emission}.}
    \State{Optimize $\gL_{\mathrm{ELBO}}(\theta,\phi)$.}
\Until{end of training.}
\end{algorithmic}
\end{algorithm}

\subsection{Stable Implementation}
\label{sec:stable-implementation}

A naive implementation of the numerical integration of ODEs (Eqs.~\ref{eq:linearization-approx} and \ref{eq:smooth-linearization-approx}) and other operations \eqrefp{eq:update-step} results in unstable training and crashing due to violation of the positive definite constraint for the covariance matrices. Commonly employed tricks such as symmetrization,
$
\rmP = (\rmP + \rmP^\top)/2,
$
and addition of a small positive number ($\epsilon$) to the diagonal elements,
$
\rmP = \rmP + \epsilon\rmI,
$
did not solve these training issues. Therefore, we implemented our algorithms in terms of square root (Cholesky) factors, which proved critical to the stable training of \ourmodel. Several square root factors' based inference algorithms have been previously proposed~\citep[Ch.~12]{zonov2019kalman,jorgensen2007computationally,kailath2000linear}. In the following, we discuss our implementation which is based on \citet{zonov2019kalman}. Further discussion on implementation stability, particularly in the case of non-linear dynamics, can be found in Appendix~\ref{app:stable-implementation}.

We begin with a lemma that shows that the square root factor of the sum of two matrices with square root factors can be computed using $\mathrm{QR}$ decomposition. 
\begin{restatable}{lemma}{sumofsqrts}
\label{lemma:sum-sqrt-factors}
Let $\rmA$ and $\rmB$ be two $n \times n$ matrices with square root factors $\rmA^{1/2}$ and $\rmB^{1/2}$, respectively. The matrix $\rmC = \rmA + \rmB$ also has a square root factor, $\rmC^{1/2}$, given by
\begin{equation*}
    \Theta, \begin{bmatrix}\rmC^{1/2} & \mathbf{0}_{n \times n}\end{bmatrix}^\top = \mathrm{QR}\left(\begin{bmatrix}\rmA^{1/2} & \rmB^{1/2}\end{bmatrix}^\top\right),
\end{equation*}
where $\Theta$ is the orthogonal $\mathrm{Q}$ matrix given by $\mathrm{QR}$ decomposition and $\mathbf{0}_{n \times n}$ is an $n \times n$ matrix of zeros.
\end{restatable}

\paragraph{Prediction step.} The solution of matrix differential equations of the form in \eqref{eq:cov-predict} --- called Lyapunov differential equations --- over $[t_0, t_1]$ is given by~\citep[Corollary~1.1.6]{abou2012matrix}
\begin{equation}
    \rmP_{t_1} = \bm{\Phi}_{t_1}\rmP_{t_0}\bm{\Phi}_{t_1}^\top + \int_{t_0}^{t_1}\bm{\Phi}_{t}\rmD_t\bm{\Phi}_{t}^\top dt,\label{eq:lyapunov-soln}
\end{equation}
where $\bm{\Phi}_t$, called the fundamental matrix, is defined by
\begin{equation}
    \frac{d\bm{\Phi}_{t}}{dt} = \rmF_{\rvz}(\rvm_t, t)\bm{\Phi}_{t} \:\: \text{and} \:\: \bm{\Phi}_{t_0} = \rmI.\label{eq:fundamental-ode}
\end{equation}
This initial value problem can be solved using an off-the-shelf ODE solver. Let $\{\tilde{\bm{\Phi}}_{1} = \rmI, \tilde{\bm{\Phi}}_{2}, \dots, \tilde{\bm{\Phi}}_{n}\}$ be intermediate solutions of \eqref{eq:fundamental-ode} given by an ODE solver with step size $\eta$, \eqref{eq:lyapunov-soln} can be approximated as
\begin{align}
    \rmP_{t_1} &\approx \tilde{\bm{\Phi}}_{n}\rmP_{t_0}\tilde{\bm{\Phi}}_{n}^\top\nonumber\\
    &+ \frac{\eta}{2}\left(\tilde{\bm{\Phi}}_{1}\rmD_{1}\tilde{\bm{\Phi}}_{1}^\top + 2\tilde{\bm{\Phi}}_{2}\rmD_{2}\tilde{\bm{\Phi}}_{2}^\top + \dots + \tilde{\bm{\Phi}}_{n}\rmD_{n}\tilde{\bm{\Phi}}_{n}^\top\right).\label{eq:cov-predict-approx}
\end{align}
The additions in \eqref{eq:cov-predict-approx} are performed using Lemma \ref{lemma:sum-sqrt-factors} with square root factors $\tilde{\bm{\Phi}}_{n}\rmP_{t_0}^{1/2}$ and $\{\tilde{\bm{\Phi}}_{j}\rmD_{j}^{1/2}\}_{j=1}^{n}$.  

\paragraph{Update step.} Using similar arguments as in the proof of Lemma \ref{lemma:sum-sqrt-factors} (cf. Appendix~\ref{app:stable-implementation} for details), the update step \eqrefp{eq:update-step} can be performed by the $\mathrm{QR}$ decomposition of the square root factor
\begin{equation}
    \begin{bmatrix}
    \rmR^{1/2} & \rmH(\rmP_k^{-})^{1/2}\\
    \mathbf{0}_{m \times d} & (\rmP_k^{-})^{1/2}
    \end{bmatrix}^\top.\label{eq:update-sqrt-factor}
\end{equation}
Let
$
    \begin{bmatrix}
    \rmX & \mathbf{0}\\
    \rmY & \rmZ
    \end{bmatrix}^\top
$
be the upper triangular $\mathrm{R}$ matrix obtained from the $\mathrm{QR}$ decomposition of (\ref{eq:update-sqrt-factor}). The square root factor of the updated covariance matrix, $\rmP_k^{1/2}$, and the Kalman gain matrix, $\rmK_k$, are then given by $\rmP_k^{1/2} = \rmZ$ and $\rmK_k = \rmY\rmX^{-1}$, respectively.
\vspace{-0.5em}
\section{Related Work}
\label{sec:related-work}
\paragraph{ODE-based models.} Since the introduction of the NeuralODE~\citep{chen2018neural}, various models based on neural ODEs have been proposed for continuous-time modeling of time series. LatentODE~\citep{rubanova2019latent} encodes the entire context window into an initial state using an encoder (e.g., ODE-RNN) and uses a NeuralODE to model the latent dynamics. ODE2VAE~\citep{yildiz2019ode2vae} decomposes the latent state into position and velocity components to explicitly model the acceleration and parameterize the ODE dynamics with Bayesian neural networks, thus accommodating uncertainty. Nevertheless, both models lack a mechanism to update the latent state based on new observations. To address this limitation, NeuralCDE~\citep{kidger2020neural} incorporates techniques from rough path theory to control the latent state using observations. Conversely, GRU-ODE-B~\citep{de2019gru} and NJ-ODE~\citep{herrera2020neural} combine neural ODEs with a Bayesian-inspired update step, enabling the incorporation of new observations. Unlike prior models, \ourmodel\ incorporates observations via a principled Bayesian update and disentangles recognition from dynamics using auxiliary variables.
\vspace{-1em}
\paragraph{SDE-based models.} LatentSDE~\citep{li2020scalable} uses a posterior SDE in the latent space to infer the latent dynamics together with a prior (generative) SDE in a variational setup. \citet{solin2021scalable} proposed a variant of LatentSDE trained by exploiting the Gaussian assumed density approximation of the non-linear SDE. VSDN~\citep{liu2020learning} uses ODE-RNNs to provide historical information about the time series to the SDE drift and diffusion functions. These models rely on posterior SDEs to infer the dynamics, and new observations are incorporated in an ad-hoc manner. This approach can potentially lead to discrepancies between the posterior and generative dynamics, as well as non-Markovian state spaces. In contrast, \ourmodel\ employs stochastic Markovian dynamics, incorporates observations through principled Bayesian updates, and performs continuous-discrete Bayesian inference for the state variables (dynamics), eliminating the need for posterior SDEs.
\vspace{-1em}
\paragraph{State space models.} Several prior works~\citep{chung2015recurrent,krishnan2015deep,karl2016deep,krishnan2017structured,doerr2018probabilistic} have proposed SSM-like models for discrete-time sequential data, trained via amortized variational inference. Unlike \ourmodel, these models approximate sequential Bayesian inference (i.e., filtering and smoothing) via deterministic RNNs and are limited to the discrete time setting. More recently, deterministic linear SSMs~\citep{gu2021efficiently,zhang2023effectively}, featuring specific transition matrices~\cite{gu2020hippo}, have been introduced as components for sequence modeling. In contrast, our work proposes a general probabilistic continuous-discrete SSM that supports locally-linear and non-linear dynamics.

The combination of Bayesian inference for a subset of latent variables and amortized inference for others has been previously explored in SSMs.
SNLDS~\citep{dong2020collapsed} and REDSDS~\citep{ansari2021deep} perform amortized inference for the states and exact inference for discrete random variables (switches and duration counts) in switching SSMs. 
KVAE~\citep{fraccaro2017disentangled}, EKVAE~\citep{klushyn2021latent} and ARSGLS~\citep{kurle2020deep} introduce auxiliary variables and perform classic Bayesian filtering and smoothing for the state variables, similar to \ourmodel. However, these models utilize specific parameterizations of state dynamics and operate on discrete-time sequential data. On the other hand, our proposed framework presents a general approach for continuous-time modeling of irregularly-sampled time series, allowing for multiple possible parameterizations of the dynamics.

The continuous recurrent unit (CRU)~\cite{schirmer2022modeling} is the most closely related model to \ourmodel, as it also employs continuous-discrete inference. However, there are notable distinctions between CRU and our work: (i) we propose a general framework offering multiple possible parameterizations of dynamics, while CRU focuses on specific locally-linear dynamics that are time-invariant between observed timesteps, unlike \ctssmll, (ii) \ourmodel\ serves as an \emph{unconditional} generative model, which fundamentally differs from CRU's conditional training for downstream tasks, and (iii) \ourmodel\ utilizes continuous-time smoothing for imputation by incorporating future information through the backward (smoothing) pass, whereas CRU solely relies on the forward (filtering) pass.

\vspace{-0.5em}
\section{Experiments}
\vspace{-0.5em}
In this section, we present empirical results on time series imputation and forecasting tasks. Our primary focus was to investigate the models' ability to capture the underlying dynamics of the time series, gauged by the accuracy of long-term forecasts beyond the training context.
We experimented with the three variants of our model described in Section \ref{sec:model-formulation}: \ctssmlti, \ctssmnl, and \ctssmll.
Our main baselines were LatentODE and LatentSDE, two popular continuous-time latent variable models with deterministic and stochastic dynamics, respectively.
We also compared \ourmodel\ against several other baselines for individual experiments. We first discuss experiment results on the low-dimensional bouncing ball and damped pendulum datasets, then move to higher dimensional settings: walking sequences from the CMU Motion Capture (MoCap) dataset, the USHCN daily climate dataset, and two 32x32 dimensional video datasets (Box and Pong). Our code is available at \url{https://github.com/clear-nus/NCDSSM}.

\begin{table*}[ht]
	\scriptsize
	\centering
	\caption{Imputation and forecasting results for bouncing ball and damped pendulum datasets averaged over 50 sample trajectories. Mean \textpm\ standard deviation are computed over 5 independent runs.}
    \label{tab:low-dim-results}
	\resizebox{\textwidth}{!}{\begin{tabular}{clccccccc}
		\toprule
        \multirow{2}{*}{Dataset} & \multirow{2}{*}{Model}   &     \multicolumn{3}{c}{Imputation MSE ($\downarrow$) (\% Missing)}  & \multicolumn{4}{c}{Forecast MSE ($\downarrow$) (\% Missing)}\\
              \cmidrule(lr){3-5} \cmidrule(lr){6-9}
		&  &  30\%  &      50\%  & 80\% & 0\%  &   30\%  &      50\%  & 80\% \\
        
		\midrule
		\multirow{6}{*}[-0.4ex]{\rotatebox{90}{\parbox[c]{1.5cm}{\centering Bouncing Ball}}} & LatentODE~{\scriptsize\citep{rubanova2019latent}}  & 0.007 \textpm\ 0.000 &	0.008 \textpm\ 0.001 &	0.011 \textpm\ 0.000 & 0.386 \textpm\ 0.025	& 0.489 \textpm\ 0.133 &	0.422 \textpm\ 0.053 &	0.412 \textpm\ 0.048 \\
        & LatentSDE~{\scriptsize\citep{li2020scalable}}  & \textbf{0.006 \textpm\ 0.000} &	0.007 \textpm\ 0.000 &	0.011 \textpm\ 0.001 & 0.408 \textpm\ 0.043	& 1.209 \textpm\ 1.115 &	1.567 \textpm\ 2.263 &	0.352 \textpm\ 0.077 \\
        & GRUODE-B~{\scriptsize\citep{de2019gru}}  & 0.017 \textpm\ 0.001 &	0.026 \textpm\ 0.010 &	0.051 \textpm\ 0.003 & 0.868 \textpm\ 0.103	& 0.805 \textpm\ 0.315 &	0.856 \textpm\ 0.394 &	0.445 \textpm\ 0.182 \\
        \cmidrule{2-9}
        & \ctssmlti\  & 0.020 \textpm\ 0.001 &	0.026 \textpm\ 0.001 &	0.067 \textpm\ 0.002 & 0.592 \textpm\ 0.106	& 0.557 \textpm\ 0.014 &	0.556 \textpm\ 0.025 &	0.555 \textpm\ 0.022 \\
        & \ctssmnl\  & \textbf{0.006 \textpm\ 0.000} &	\textbf{0.006 \textpm\ 0.000} &	\textbf{0.007 \textpm\ 0.000} & \textbf{0.037 \textpm\ 0.018}	& 0.036 \textpm\ 0.007 &	\textbf{0.041 \textpm\ 0.007} &	0.115 \textpm\ 0.029 \\
        & \ctssmll\  & \textbf{0.006 \textpm\ 0.000} &	\textbf{0.006 \textpm\ 0.000} &	0.008 \textpm\ 0.001 & \textbf{0.037 \textpm\ 0.028}	& \textbf{0.034 \textpm\ 0.016} &	0.049 \textpm\ 0.034 &	\textbf{0.076 \textpm\ 0.017} \\
        \midrule
        \multirow{6}{*}[-0.4ex]{\rotatebox{90}{\parbox[c]{1.5cm}{\centering Damped \mbox{Pendulum}}}} & LatentODE~{\scriptsize\citep{rubanova2019latent}}  & 0.151 \textpm\ 0.002 &	0.155 \textpm\ 0.002 &	0.206 \textpm\ 0.013 & 0.097 \textpm\ 0.042 &	0.117 \textpm\ 0.001	& 0.119 \textpm\ 0.001 &	0.148 \textpm\ 0.007 \\
        & LatentSDE~{\scriptsize\citep{li2020scalable}} & 0.092 \textpm\ 0.076 &	0.148 \textpm\ 0.001 &	0.229 \textpm\ 0.001 & 0.046 \textpm\ 0.046 &	0.084 \textpm\ 0.058	& 0.147 \textpm\ 0.020 &	0.357 \textpm\ 0.096 \\
        & GRUODE-B~{\scriptsize\citep{de2019gru}} & 0.015 \textpm\ 0.001 &	0.023 \textpm\ 0.003 &	0.064 \textpm\ 0.003 & 0.244 \textpm\ 0.107 &	0.424 \textpm\ 0.617	& 0.124 \textpm\ 0.088 &	0.037 \textpm\ 0.036 \\
        \cmidrule{2-9}
        & \ctssmlti\  & 0.036 \textpm\ 0.001 &	0.057 \textpm\ 0.001 &	0.120 \textpm\ 0.002 & 0.282 \textpm\ 0.084 &	1.017 \textpm\ 1.363	& 1.527 \textpm\ 1.440 &	0.231 \textpm\ 0.050 \\
        & \ctssmnl\ & \textbf{0.008 \textpm\ 0.000} &	\textbf{0.011 \textpm\ 0.000} &	\textbf{0.033 \textpm\ 0.002} & \textbf{0.011 \textpm\ 0.004} &	0.011 \textpm\ 0.003	& \textbf{0.012 \textpm\ 0.003} &	\textbf{0.034 \textpm\ 0.019} \\
        & \ctssmll\ & \textbf{0.008 \textpm\ 0.000} &	\textbf{0.011 \textpm\ 0.000} &	0.037 \textpm\ 0.003 & 0.025 \textpm\ 0.030 &	\textbf{0.010 \textpm\ 0.001}	& 0.020 \textpm\ 0.008 &	0.055 \textpm\ 0.007 \\
		\bottomrule
	\end{tabular}}
    \vspace{-1em}
\end{table*}

\vspace{-0.5em}
\subsection{Bouncing Ball and Damped Pendulum}
\label{sec:exp-low-dim}
\vspace{-0.5em}
The bouncing ball and damped pendulum datasets have known ground truth dynamics, which facilitates quality assessment of the dynamics learned by a given model.
For details on these datasets, please refer to Appendix~\ref{app:datasets}.
In brief, the univariate bouncing ball dataset exhibits piecewise-linear dynamics, whilst bivariate damped pendulum dataset~\citep{karl2016deep,kurle2020deep} exhibits non-linear latent dynamics.

\begin{figure}
\includegraphics[width=1\linewidth]{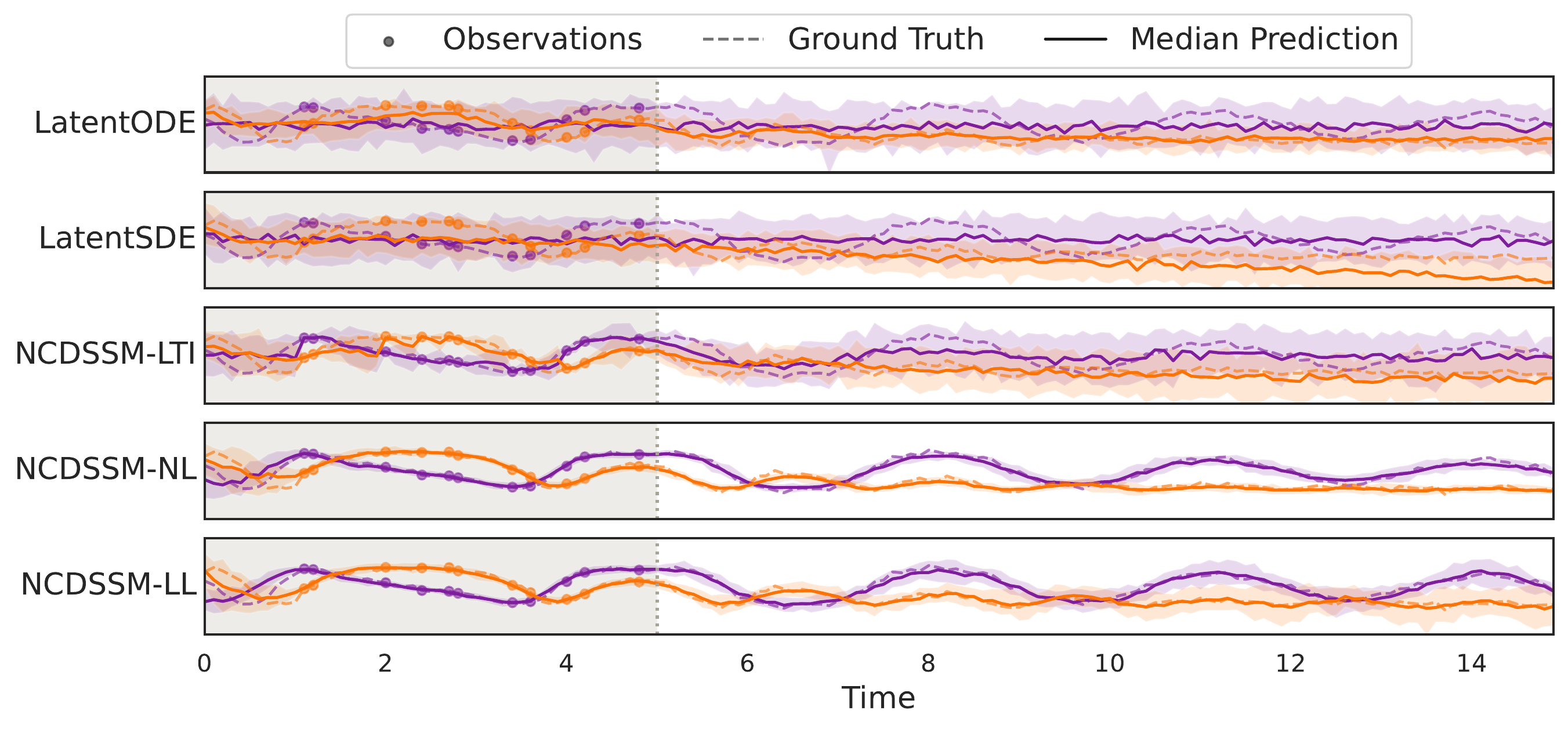}
\vspace{-2em}
\caption{Predictions from different models on the damped pendulum dataset in the 80\% missing data setting. The ground truth is shown using dashed lines with observed points in the context window (gray shaded region) shown as filled circles. The vertical dashed gray line marks the beginning of the forecast horizon. Solid lines indicate median predictions with 90\% prediction intervals shaded around them. The purple and orange colors indicate observation dimensions. \ctssmnl\ and \ctssmll\ are significantly better at forecasting compared to the baselines.}
\label{fig:low-dim-forecasts}
\vspace{-2em}
\end{figure}

\begin{table}
	\scriptsize
	\centering
    \vspace{-1em}
	\caption{Forecasting results for the CMU MoCap walking dataset averaged over 50 sample trajectories with 95\% prediction interval based on the $t$-statistic in parentheses. $^\dagger$Baseline results from \citet{solin2021scalable}.}
    \label{tab:mocap-results}
    \resizebox{\columnwidth}{!}{\begin{tabular}{lrr}
		\toprule
        \multirow{2}{*}{Model} & \multicolumn{2}{c}{MSE ($\downarrow$)}\\
        \cmidrule{2-3}
         & \multicolumn{1}{c}{$^\dagger$Setup~1} & \multicolumn{1}{c}{Setup~2}\\
        \midrule
        {np}ODE~{\scriptsize\citep{heinonen2018learning}} & 22.96 & \multicolumn{1}{c}{--} \\
        {Neural}ODE~{\scriptsize\citep{chen2018neural}} &  22.49 (0.88) & \multicolumn{1}{c}{--} \\
        ODE\textsuperscript{2}VAE-KL~{\scriptsize\citep{yildiz2019ode2vae}} & 8.09 (1.95) & \multicolumn{1}{c}{--}\\
        LatentODE~{\scriptsize\citep{rubanova2019latent}} & 5.98 (0.28) & 31.62	(0.05) \\
        LatentSDE~{\scriptsize\citep{li2020scalable}} & \textbf{4.03 (0.20)} & 9.52 (0.21) \\
        LatentApproxSDE~{\scriptsize\citep{solin2021scalable}}  &  7.55 (0.05) & \multicolumn{1}{c}{--} \\
        \midrule
        \ctssmlti\  & 13.90 (0.02) & 5.22 (0.02) \\
        \ctssmnl\  & 5.69 (0.01) & 6.73 (0.02) \\
        \ctssmll\  & 9.96 (0.01) & \textbf{4.74 (0.01)} \\
        \bottomrule
    \end{tabular}}
    \vspace{-2em}
\end{table}

We trained all the models on 10s/5s sequences (with a discretization of 0.1s) for bouncing ball/damped pendulum with 0\%, 30\%, 50\% and 80\% timesteps missing at random to simulate irregularly-sampled data. The models were evaluated on imputation of the missing timesteps and forecasts of 20s/10s beyond the training regime for bouncing ball/damped pendulum. 

Table~\ref{tab:low-dim-results} reports the imputation and forecast mean squared error (MSE) for different missing data settings. 
In summary, the \ourmodel\ models with non-linear and locally-linear dynamics (\ctssmnl\ and \ctssmll) perform well across datasets, settings, and random initializations, significantly outperforming the baselines. Furthermore, for these low-dimensional datasets, learning latent representations in the form of auxiliary variables is not required and we can set the recognition and emission functions in \eqref{eq:auxiliary-inference} and \eqref{eq:observation-emission} to identity functions. This results in \ourmodel\ models requiring 2-5 times fewer parameters than LatentODE and LatentSDE 
(cf. Table~\ref{tab:parameter-comparison} in the Appendix).

Fig.~\ref{fig:low-dim-forecasts} shows example predictions from the best performing run of every model for 80\% missing data for the pendulum (cf.~Appendix~\ref{app:additional-results} for other settings). \ctssmnl\ and \ctssmll\ generates far better predictions both inside and outside the context window compared to the baselines. Ordinary least squares (OLS) goodness-of-fit results in Table~\ref{tab:ols-pendulum} (Appendix) suggest that this performance can be attributed to our models having learnt the correct dynamics; latent states from \ctssmnl\ and \ctssmll\ are highly correlated with the ground truth angle and angular velocity for all missingness scenarios. In other words, the models have learnt a Markovian state space which is informative about the dynamics at a specific time. 

\vspace{-0.5em}
\subsection{CMU Motion Capture (Walking)}
\label{sec:exp-mocap}
\vspace{-0.5em}

This dataset comprises walking sequences of subject 35 from the CMU MoCap database containing joint angles of subjects performing everyday activities. We used a preprocessed version of the dataset from \citet{yildiz2019ode2vae} that has 23 50-dimensional sequences of length 300.

We tested the models under two setups. Setup 1~~\citep{yildiz2019ode2vae,li2020scalable,solin2021scalable}  involves training on complete 300 timestep sequences from the training set and using only the first 3 timesteps as context to predict the remaining 297 timesteps during test time. Although challenging, this setup does not evaluate the model's performance beyond the training context.  %
Thus, we propose Setup 2 in which we train the model only using the first 200 timesteps. During test time, we give the first 100 timesteps as context and predict the remaining 200 timesteps.

The forecast MSE results for both setups are reported in Table~\ref{tab:mocap-results}. \ctssmnl\ performs better than all baselines except LatentSDE on Setup~1 while \ourmodel\ models perform significantly better than baselines on Setup~2. This showcases \ourmodel's ability to correctly model the latent dynamics, aiding accurate long-term predictions beyond the training context. 

\vspace{-0.8em}
\subsection{USHCN Climate Indicators}
\label{sec:exp-ushcn}
\vspace{-0.5em}
We evaluated the models on the United States Historical Climatology Network (USHCN) dataset that comprises measurements of five climate indicators across the United States. The preprocessed version of this dataset from \citet{de2019gru} contains sporadic time series (i.e., with measurements missing \emph{both} over the time and feature axes) from 1,114 meteorological stations over 4 years. Following \citet{de2019gru}, we trained the models on sequences from the training stations and evaluated them on the task of predicting the next 3 measurements given the first 3 years as context from the held-out test stations. The results in Table~\ref{tab:ushcn-results} show that \ctssmnl\ outperforms all the baselines with \ctssmlti\ and \ctssmll\ performing better than most of the baselines. 
\vspace{-0.5em}
\subsection{Pymunk Physical Environments}
\label{sec:exp-pymunk}
\vspace{-0.5em}
Finally, we evaluated the models on two high-dimensional (video) datasets of physical environments used in \citet{fraccaro2017disentangled}, simulated using the Pymunk Physics engine~\cite{Blomqvist_Pymunk_2022}: Box and Pong. The box dataset consists of videos of a ball moving in a 2-dimensional box and the pong dataset consists of videos of a Pong-like environment where two paddles move to keep a ball in the frame at all times. Each frame is a 32x32 binary image. 

\begin{figure*}[ht]
    \centering
\includegraphics[width=1.0\linewidth]{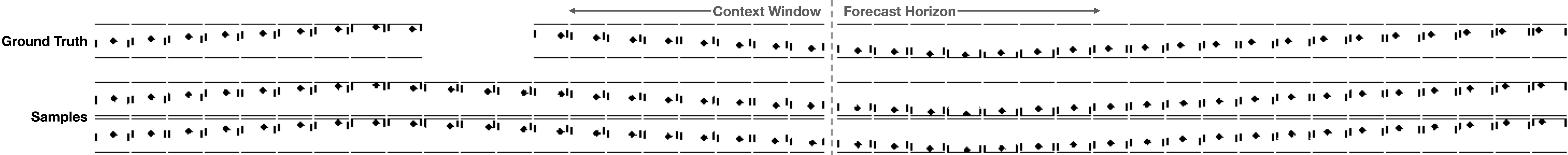}
    \vspace{-2em}
    \caption{Sample predictions from \ctssmnl\ on the Pong dataset. The top row is the ground truth with some missing observations in the context window. The next two rows show trajectories sampled from \ctssmnl{} upto 20 forecast steps. \ctssmnl\ is able to both impute and forecast accurately. Best viewed zoomed-in on a computer. More examples in Appendix~\ref{app:additional-results}.}
    \label{fig:pymunk-predictions}
    \vspace{-1.5em}
\end{figure*}

\begin{table}
	\scriptsize
	\centering
    \vspace{-1em}
	\caption{Forecasting results for the USHCN climate dataset. Mean \textpm\ standard deviation are computed over 5 folds as described in \citet{de2019gru}. $^\dagger$Results from \citet{de2019gru}. $^\ddagger$Results from \citet{liu2020learning}.}
    \label{tab:ushcn-results}
    \begin{tabular}{lc}
		\toprule
        Model & \multicolumn{1}{c}{MSE ($\downarrow$)}\\
        \midrule
        $^\dagger$NeuralODE-VAE~{\scriptsize\citep{chen2018neural}} & 0.83 \textpm\ 0.10 \\
        $^\dagger$SequentialVAE~{\scriptsize\citep{krishnan2015deep}} & 0.83 \textpm\ 0.07 \\
        $^\dagger$GRU-D~{\scriptsize\citep{che2018recurrent}} & 0.53 \textpm\ 0.06 \\
        $^\dagger$T-LSTM~{\scriptsize\citep{baytas2017patient}} & 0.59 \textpm\ 0.11 \\
        $^\dagger$GRUODE-B{\scriptsize~\citep{de2019gru}} & 0.43 \textpm\ 0.07\\
        $^\ddagger$ODE-RNN~{\scriptsize\citep{rubanova2019latent}} & 0.39 \textpm\ 0.06 \\
        $^\ddagger$LatentODE~{\scriptsize\citep{rubanova2019latent}} & 0.77 \textpm\ 0.09 \\
        $^\ddagger$LatentSDE~{\scriptsize\citep{li2020scalable}} & 0.74 \textpm\ 0.11 \\
        $^\ddagger$VSDN-F (IWAE)~{\scriptsize\citep{liu2020learning}} & 0.37 \textpm\ 0.06 \\
        \midrule
        \ctssmlti\  & 0.38 \textpm\ 0.07 \\
        \ctssmnl\  & \textbf{0.34 \textpm\ 0.06} \\
        \ctssmll\  & 0.37 \textpm\ 0.06 \\
        \bottomrule
    \end{tabular}
    \vspace{-2.2em}
\end{table}
\begin{table}
	\scriptsize
	\centering
	\caption{Forecasting results for the Box and Pong datasets averaged over 16 sample trajectories.}
    \label{tab:pymunk-results}
    \begin{tabular}{lcc}
		\toprule
        \multirow{2}{*}{Model} & \multicolumn{2}{c}{EMD ($\downarrow$)}\\
        \cmidrule{2-3}
         & Box & Pong\\
        \midrule
        LatentODE~{\scriptsize\citep{rubanova2019latent}} & 1.792 & 4.543 \\
        LatentSDE~{\scriptsize\citep{li2020scalable}} & 1.925 & 3.505 \\
        \midrule
        \ctssmlti\  & 1.685 & 3.265 \\
        \ctssmnl\  & 0.692 & \textbf{1.714} \\
        \ctssmll\  & \textbf{0.632} & 1.891 \\
        \bottomrule
    \end{tabular}
    \vspace{-2.2em}
\end{table}

We trained the models on sequences of 20 frames with 20\% of these frames randomly dropped. At test time, the models were evaluated on forecasts of 40 frames beyond the training context. For evaluation, we treat each image as a probability distribution on the XY-plane and report the earth mover's distance (EMD) between the ground truth and predicted images, averaged over the forecast horizon, in Table~\ref{tab:pymunk-results}. \ctssmnl\ and \ctssmll\ significantly outperform baseline models on both box and pong datasets. Fig.~\ref{fig:pymunk-emd} (Appendix) shows the variation of EMD against time for different models. In the context window (0-2s), all models have EMD close to 0; however, in the forecast horizon (2-6s), the EMD rises rapidly and irregularly for LatentODE and LatentSDE but does so gradually for \ctssmnl\ and \ctssmll. This indicates that the dynamics models learned by \ctssmnl\ and \ctssmll\ are both accurate and robust. 

Qualitatively, both \ctssmll\ and \ctssmnl\ correctly impute the missing frames and the forecasts generated by them are similar to ground truth. Fig.~\ref{fig:pymunk-predictions} shows sample predictions for the pong dataset generated by \ctssmnl. In contrast, other models only impute the missing frames correctly, failing to generate accurate forecasts (cf.~Appendix~\ref{app:additional-results}).  
\vspace{-0.8em}
\section{Discussion}
\label{sec:discussion}
\vspace{-0.5em}
\paragraph{Choice of dynamics.} The selection of latent dynamics heavily relies on the dataset and the specific problem at hand. Nevertheless, we offer some general guidelines based on our experiments and observations. The linear time-invariant (LTI) dynamics are well-suited when the time series exhibit approximate linearity or when fast inference is crucial (as the predict step can be analytically computed using matrix exponentials). The locally-linear (LL) dynamics performs exceptionally well out of the box and is highly desirable for achieving quick, high-quality results. It requires minimal tuning and regularization since the drift parameterization is straightforward through the $K$ base matrices, which control the dynamics' flexibility. On the other hand, the non-linear model demonstrates superior performance in most scenarios but necessitates careful parameterization, specifically in selecting the drift network, and rigorous regularization. We delve into the parameterization and regularization of non-linear dynamics in Appendix~\ref{app:stable-implementation} and anticipate that our findings will prove valuable for non-linear models beyond \ctssmnl.
\vspace{-1.3em}
\paragraph{Time complexity.} The time complexities of the predict and update steps primarily depend on drift function evaluation/matrix multiplication and matrix inversion, respectively. As a result, the overall complexity of the filtering process is $O(N_\mathrm{int} C_\mathrm{int}(N_\mathrm{dfe}) + T_\mathrm{obs} h^3)$, where $N_\mathrm{int}$ represents the number of integration steps, $C_\mathrm{int}$ denotes the cost of a single integration step (which is influenced by the number of drift function evaluations, $N_\mathrm{dfe}$), $T_\mathrm{obs}$ corresponds to the number of observed timesteps, and $h$ represents the dimensionality of the auxiliary variable. The first term aligns with the cost incurred by ODE-based models. The NCDSSM incurs an additional overhead of $O(T_\mathrm{obs} h^3)$ due to the Bayesian update. It is worth noting that although the complexity $h^3$ (or approximately $h^{2.4}$, depending on the chosen matrix inversion algorithm) exhibits poor scaling with respect to $h$, the assumption is made that the auxiliary variables are of low dimensionality.
\vspace{-1.3em}
\paragraph{Limitations.} As discussed above, the update step in NCDSSM incurs an additional computational cost of $O(T_\mathrm{obs} h^3)$ compared to ODE-based models, resulting in slower training and inference. In this study, our primary focus was on ensuring model stability and accurate predictions. Nonetheless, we acknowledge that several optimizations can be explored to enhance the computational efficiency of inference, e.g., by side-stepping explicit matrix inversion in the update step and by choosing adaptive solvers that allow for larger step sizes. We defer these investigations to future research.

Furthermore, the linearization of the drift function within our Gaussian assumed density approximation may impose limitations on the expressiveness of the non-linear dynamics. Although this approximation outperforms alternative types of dynamics in our experimental evaluations, we believe that further improvements are attainable. For instance, employing sigma point approximations through Gauss–Hermite integration or Unscented transformation~\citep[Ch.~9]{sarkka2019applied} could enhance the modeling accuracy and flexibility.
\vspace{-0.8em}
\section{Conclusion}
\label{sec:conclusion}
\vspace{-0.5em}
In this work, we proposed a model for continuous-time modeling of irregularly-sampled time series. \ourmodel\ improves continuous-discrete SSMs with neural network-based parameterizations of dynamics, and modern inference and learning techniques. Through the introduction of auxiliary variables, \ourmodel\ enables efficient modeling of high-dimensional time series while allowing accurate continuous-discrete Bayesian inference of the dynamic states. Experiments on a variety of low- and high-dimensional datasets show that \ourmodel\ outperforms existing models on time series imputation and forecasting tasks.
\vspace{-0.8em}
\section*{Acknowledgements}
\vspace{-0.5em}
This research is supported by the National Research Foundation Singapore and DSO National Laboratories under the AI Singapore Programme (AISG Award No: AISG2-RP-2020-016). We would like to express our gratitude to Richard Kurle, Fabian Falck, Alexej Klushyn, and Marcel Kollovieh for their valuable discussions and feedback. We would also like to extend our appreciation to the anonymous reviewers whose insightful suggestions helped enhance the clarity of the manuscript.

\bibliography{main}

\begin{thebibliography}{42}
\providecommand{\natexlab}[1]{#1}
\providecommand{\url}[1]{\texttt{#1}}
\expandafter\ifx\csname urlstyle\endcsname\relax
  \providecommand{\doi}[1]{doi: #1}\else
  \providecommand{\doi}{doi: \begingroup \urlstyle{rm}\Url}\fi

\bibitem[Abou-Kandil et~al.(2012)Abou-Kandil, Freiling, Ionescu, and
  Jank]{abou2012matrix}
Abou-Kandil, H., Freiling, G., Ionescu, V., and Jank, G.
\newblock \emph{Matrix {Riccati} equations in control and systems theory}.
\newblock Birkh{\"a}user, 2012.

\bibitem[Anderson(1972)]{anderson1972fixed}
Anderson, B.~D.
\newblock Fixed interval smoothing for nonlinear continuous time systems.
\newblock \emph{Information and Control}, 20\penalty0 (3):\penalty0 294--300,
  1972.

\bibitem[Ansari et~al.(2021)Ansari, Benidis, Kurle, Turkmen, Soh, Smola, Wang,
  and Januschowski]{ansari2021deep}
Ansari, A.~F., Benidis, K., Kurle, R., Turkmen, A.~C., Soh, H., Smola, A.~J.,
  Wang, B., and Januschowski, T.
\newblock Deep explicit duration switching models for time series.
\newblock \emph{Advances in Neural Information Processing Systems}, 34, 2021.

\bibitem[Baytas et~al.(2017)Baytas, Xiao, Zhang, Wang, Jain, and
  Zhou]{baytas2017patient}
Baytas, I.~M., Xiao, C., Zhang, X., Wang, F., Jain, A.~K., and Zhou, J.
\newblock Patient subtyping via time-aware {LSTM} networks.
\newblock In \emph{Proceedings of the 23rd ACM SIGKDD international conference
  on knowledge discovery and data mining}, pp.\  65--74, 2017.

\bibitem[Blomqvist(2022)]{Blomqvist_Pymunk_2022}
Blomqvist, V.
\newblock {Pymunk}, 11 2022.
\newblock URL \url{https://pymunk.org}.

\bibitem[Che et~al.(2018)Che, Purushotham, Cho, Sontag, and
  Liu]{che2018recurrent}
Che, Z., Purushotham, S., Cho, K., Sontag, D., and Liu, Y.
\newblock Recurrent neural networks for multivariate time series with missing
  values.
\newblock \emph{Scientific reports}, 8\penalty0 (1):\penalty0 1--12, 2018.

\bibitem[Chen et~al.(2018)Chen, Rubanova, Bettencourt, and
  Duvenaud]{chen2018neural}
Chen, R.~T., Rubanova, Y., Bettencourt, J., and Duvenaud, D.~K.
\newblock Neural ordinary differential equations.
\newblock \emph{Advances in neural information processing systems}, 31, 2018.

\bibitem[Chung et~al.(2015)Chung, Kastner, Dinh, Goel, Courville, and
  Bengio]{chung2015recurrent}
Chung, J., Kastner, K., Dinh, L., Goel, K., Courville, A.~C., and Bengio, Y.
\newblock A recurrent latent variable model for sequential data.
\newblock \emph{Advances in neural information processing systems}, 28, 2015.

\bibitem[De~Brouwer et~al.(2019)De~Brouwer, Simm, Arany, and Moreau]{de2019gru}
De~Brouwer, E., Simm, J., Arany, A., and Moreau, Y.
\newblock {GRU-ODE-Bayes}: Continuous modeling of sporadically-observed time
  series.
\newblock \emph{Advances in neural information processing systems}, 32, 2019.

\bibitem[Doerr et~al.(2018)Doerr, Daniel, Schiegg, Duy, Schaal, Toussaint, and
  Sebastian]{doerr2018probabilistic}
Doerr, A., Daniel, C., Schiegg, M., Duy, N.-T., Schaal, S., Toussaint, M., and
  Sebastian, T.
\newblock Probabilistic recurrent state-space models.
\newblock In \emph{International Conference on Machine Learning}, pp.\
  1280--1289. PMLR, 2018.

\bibitem[Dong et~al.(2020)Dong, Seybold, Murphy, and Bui]{dong2020collapsed}
Dong, Z., Seybold, B., Murphy, K., and Bui, H.
\newblock Collapsed amortized variational inference for switching nonlinear
  dynamical systems.
\newblock In \emph{International Conference on Machine Learning}, pp.\
  2638--2647. PMLR, 2020.

\bibitem[Flamary et~al.(2021)Flamary, Courty, Gramfort, Alaya, Boisbunon,
  Chambon, Chapel, Corenflos, Fatras, Fournier, Gautheron, Gayraud, Janati,
  Rakotomamonjy, Redko, Rolet, Schutz, Seguy, Sutherland, Tavenard, Tong, and
  Vayer]{flamary2021pot}
Flamary, R., Courty, N., Gramfort, A., Alaya, M.~Z., Boisbunon, A., Chambon,
  S., Chapel, L., Corenflos, A., Fatras, K., Fournier, N., Gautheron, L.,
  Gayraud, N.~T., Janati, H., Rakotomamonjy, A., Redko, I., Rolet, A., Schutz,
  A., Seguy, V., Sutherland, D.~J., Tavenard, R., Tong, A., and Vayer, T.
\newblock Pot: Python optimal transport.
\newblock \emph{Journal of Machine Learning Research}, 22\penalty0
  (78):\penalty0 1--8, 2021.
\newblock URL \url{http://jmlr.org/papers/v22/20-451.html}.

\bibitem[Fraccaro et~al.(2017)Fraccaro, Kamronn, Paquet, and
  Winther]{fraccaro2017disentangled}
Fraccaro, M., Kamronn, S., Paquet, U., and Winther, O.
\newblock A disentangled recognition and nonlinear dynamics model for
  unsupervised learning.
\newblock \emph{arXiv preprint arXiv:1710.05741}, 2017.

\bibitem[Goldberger et~al.(2000)Goldberger, Amaral, Glass, Hausdorff, Ivanov,
  Mark, Mietus, Moody, Peng, and Stanley]{PhysioNet}
Goldberger, A.~L., Amaral, L. A.~N., Glass, L., Hausdorff, J.~M., Ivanov,
  P.~C., Mark, R.~G., Mietus, J.~E., Moody, G.~B., Peng, C.-K., and Stanley,
  H.~E.
\newblock {PhysioBank, PhysioToolkit, and PhysioNet}: Components of a new
  research resource for complex physiologic signals.
\newblock \emph{Circulation}, 101\penalty0 (23):\penalty0 e215--e220, 2000.

\bibitem[Gu et~al.(2020)Gu, Dao, Ermon, Rudra, and R{\'e}]{gu2020hippo}
Gu, A., Dao, T., Ermon, S., Rudra, A., and R{\'e}, C.
\newblock Hippo: Recurrent memory with optimal polynomial projections.
\newblock \emph{Advances in neural information processing systems},
  33:\penalty0 1474--1487, 2020.

\bibitem[Gu et~al.(2021)Gu, Goel, and R{\'e}]{gu2021efficiently}
Gu, A., Goel, K., and R{\'e}, C.
\newblock Efficiently modeling long sequences with structured state spaces.
\newblock \emph{arXiv preprint arXiv:2111.00396}, 2021.

\bibitem[Heinonen et~al.(2018)Heinonen, Yildiz, Mannerstr{\"o}m, Intosalmi, and
  L{\"a}hdesm{\"a}ki]{heinonen2018learning}
Heinonen, M., Yildiz, C., Mannerstr{\"o}m, H., Intosalmi, J., and
  L{\"a}hdesm{\"a}ki, H.
\newblock Learning unknown {ODE} models with {Gaussian} processes.
\newblock In \emph{International Conference on Machine Learning}, pp.\
  1959--1968. PMLR, 2018.

\bibitem[Herrera et~al.(2020)Herrera, Krach, and Teichmann]{herrera2020neural}
Herrera, C., Krach, F., and Teichmann, J.
\newblock Neural jump ordinary differential equations: Consistent
  continuous-time prediction and filtering.
\newblock \emph{arXiv preprint arXiv:2006.04727}, 2020.

\bibitem[Jazwinski(1970)]{jazwinski1970stochastic}
Jazwinski, A.~H.
\newblock \emph{Stochastic processes and filtering theory}.
\newblock Academic Press, 1970.

\bibitem[Jorgensen et~al.(2007)Jorgensen, Thomsen, Madsen, and
  Kristensen]{jorgensen2007computationally}
Jorgensen, J.~B., Thomsen, P.~G., Madsen, H., and Kristensen, M.~R.
\newblock A computationally efficient and robust implementation of the
  continuous-discrete extended {Kalman} filter.
\newblock In \emph{2007 American Control Conference}, pp.\  3706--3712. IEEE,
  2007.

\bibitem[Kailath et~al.(2000)Kailath, Sayed, and Hassibi]{kailath2000linear}
Kailath, T., Sayed, A.~H., and Hassibi, B.
\newblock \emph{Linear estimation}.
\newblock Prentice Hall, 2000.

\bibitem[Karl et~al.(2016)Karl, Soelch, Bayer, and Van~der Smagt]{karl2016deep}
Karl, M., Soelch, M., Bayer, J., and Van~der Smagt, P.
\newblock Deep variational {Bayes} filters: Unsupervised learning of state
  space models from raw data.
\newblock \emph{arXiv preprint arXiv:1605.06432}, 2016.

\bibitem[Kidger et~al.(2020)Kidger, Morrill, Foster, and
  Lyons]{kidger2020neural}
Kidger, P., Morrill, J., Foster, J., and Lyons, T.
\newblock Neural controlled differential equations for irregular time series.
\newblock \emph{Advances in Neural Information Processing Systems},
  33:\penalty0 6696--6707, 2020.

\bibitem[Klushyn et~al.(2021)Klushyn, Kurle, Soelch, Cseke, and van~der
  Smagt]{klushyn2021latent}
Klushyn, A., Kurle, R., Soelch, M., Cseke, B., and van~der Smagt, P.
\newblock Latent matters: Learning deep state-space models.
\newblock \emph{Advances in Neural Information Processing Systems},
  34:\penalty0 10234--10245, 2021.

\bibitem[Krishnan et~al.(2017)Krishnan, Shalit, and
  Sontag]{krishnan2017structured}
Krishnan, R., Shalit, U., and Sontag, D.
\newblock Structured inference networks for nonlinear state space models.
\newblock In \emph{Proceedings of the AAAI Conference on Artificial
  Intelligence}, volume~31, 2017.

\bibitem[Krishnan et~al.(2015)Krishnan, Shalit, and Sontag]{krishnan2015deep}
Krishnan, R.~G., Shalit, U., and Sontag, D.
\newblock Deep {Kalman} filters.
\newblock \emph{arXiv preprint arXiv:1511.05121}, 2015.

\bibitem[Kurle et~al.(2020)Kurle, Rangapuram, de~B{\'e}zenac, G{\"u}nnemann,
  and Gasthaus]{kurle2020deep}
Kurle, R., Rangapuram, S.~S., de~B{\'e}zenac, E., G{\"u}nnemann, S., and
  Gasthaus, J.
\newblock Deep rao-blackwellised particle filters for time series forecasting.
\newblock \emph{Advances in Neural Information Processing Systems}, 33, 2020.

\bibitem[Li et~al.(2020)Li, Wong, Chen, and Duvenaud]{li2020scalable}
Li, X., Wong, T.-K.~L., Chen, R.~T., and Duvenaud, D.
\newblock Scalable gradients for stochastic differential equations.
\newblock In \emph{International Conference on Artificial Intelligence and
  Statistics}, pp.\  3870--3882. PMLR, 2020.

\bibitem[Liu et~al.(2020)Liu, Xing, Yang, Wang, Shi, Jin, and
  Chen]{liu2020learning}
Liu, Y., Xing, Y., Yang, X., Wang, X., Shi, J., Jin, D., and Chen, Z.
\newblock Learning continuous-time dynamics by stochastic differential
  networks.
\newblock \emph{arXiv preprint arXiv:2006.06145}, 2020.

\bibitem[Menne et~al.(2010)Menne, Williams~Jr, and Vose]{menne2010long}
Menne, M., Williams~Jr, C., and Vose, R.
\newblock Long-term daily and monthly climate records from stations across the
  contiguous united states.
\newblock \emph{Website http://cdiac. ornl. gov/epubs/ndp/ushcn/access. html
  [accessed 23 September 2010]}, 2010.

\bibitem[Miyato et~al.(2018)Miyato, Kataoka, Koyama, and
  Yoshida]{miyato2018spectral}
Miyato, T., Kataoka, T., Koyama, M., and Yoshida, Y.
\newblock Spectral normalization for generative adversarial networks.
\newblock \emph{arXiv preprint arXiv:1802.05957}, 2018.

\bibitem[{\O}ksendal(2003)]{oksendal2003stochastic}
{\O}ksendal, B.
\newblock Stochastic differential equations.
\newblock In \emph{Stochastic differential equations}, pp.\  65--84. Springer,
  2003.

\bibitem[Rangapuram et~al.(2018)Rangapuram, Seeger, Gasthaus, Stella, Wang, and
  Januschowski]{rangapuram2018deep}
Rangapuram, S.~S., Seeger, M.~W., Gasthaus, J., Stella, L., Wang, Y., and
  Januschowski, T.
\newblock Deep state space models for time series forecasting.
\newblock \emph{Advances in neural information processing systems}, 31, 2018.

\bibitem[Rubanova et~al.(2019)Rubanova, Chen, and Duvenaud]{rubanova2019latent}
Rubanova, Y., Chen, R.~T., and Duvenaud, D.~K.
\newblock Latent ordinary differential equations for irregularly-sampled time
  series.
\newblock \emph{Advances in neural information processing systems}, 32, 2019.

\bibitem[S{\"a}rkk{\"a} \& Sarmavuori(2013)S{\"a}rkk{\"a} and
  Sarmavuori]{sarkka2013gaussian}
S{\"a}rkk{\"a}, S. and Sarmavuori, J.
\newblock Gaussian filtering and smoothing for continuous-discrete dynamic
  systems.
\newblock \emph{Signal Processing}, 93\penalty0 (2):\penalty0 500--510, 2013.

\bibitem[S{\"a}rkk{\"a} \& Solin(2019)S{\"a}rkk{\"a} and
  Solin]{sarkka2019applied}
S{\"a}rkk{\"a}, S. and Solin, A.
\newblock \emph{Applied stochastic differential equations}.
\newblock Cambridge University Press, 2019.

\bibitem[Schirmer et~al.(2022)Schirmer, Eltayeb, Lessmann, and
  Rudolph]{schirmer2022modeling}
Schirmer, M., Eltayeb, M., Lessmann, S., and Rudolph, M.
\newblock Modeling irregular time series with continuous recurrent units.
\newblock In \emph{International Conference on Machine Learning}, pp.\
  19388--19405. PMLR, 2022.

\bibitem[Solin et~al.(2021)Solin, Tamir, and Verma]{solin2021scalable}
Solin, A., Tamir, E., and Verma, P.
\newblock Scalable inference in {SDEs} by direct matching of the
  {Fokker--Planck--Kolmogorov} equation.
\newblock \emph{Advances in Neural Information Processing Systems},
  34:\penalty0 417--429, 2021.

\bibitem[Turkmen et~al.(2019)Turkmen, Wang, and
  Januschowski]{turkmen2019intermittent}
Turkmen, A.~C., Wang, Y., and Januschowski, T.
\newblock Intermittent demand forecasting with deep renewal processes.
\newblock \emph{arXiv preprint arXiv:1911.10416}, 2019.

\bibitem[Yildiz et~al.(2019)Yildiz, Heinonen, and
  Lahdesmaki]{yildiz2019ode2vae}
Yildiz, C., Heinonen, M., and Lahdesmaki, H.
\newblock {ODE2VAE}: Deep generative second order {ODEs} with {Bayesian} neural
  networks.
\newblock \emph{Advances in Neural Information Processing Systems}, 32, 2019.

\bibitem[Zhang et~al.(2023)Zhang, Saab, Poli, Dao, Goel, and
  R{\'e}]{zhang2023effectively}
Zhang, M., Saab, K.~K., Poli, M., Dao, T., Goel, K., and R{\'e}, C.
\newblock Effectively modeling time series with simple discrete state spaces.
\newblock \emph{arXiv preprint arXiv:2303.09489}, 2023.

\bibitem[Zonov(2019)]{zonov2019kalman}
Zonov, S.
\newblock Kalman filter based sensor placement for {Burgers} equation.
\newblock Master's thesis, University of Waterloo, 2019.

\end{thebibliography}
\bibliographystyle{icml2023}

\newpage
\onecolumn
\appendix
\section{Proofs}
\subsection{Proof of Lemma~\ref{lemma:sum-sqrt-factors}}
\sumofsqrts*
\begin{proof}
Our proof is based on \citet[Thm.~3.2]{zonov2019kalman}. Consider the square root factor
$$\rmY = \begin{bmatrix}\rmA^{1/2} & \rmB^{1/2}\end{bmatrix}.$$
Clearly, $\rmC = \rmY\rmY^\top$; however, we also have $\rmC = \rmY\Theta\Theta^\top\rmY^\top$, for any orthogonal matrix $\Theta$. Thus, $\rmY\Theta$ is also a square root factor of $\rmC$. Let $\Theta$ be an orthogonal matrix such that
\begin{equation}
    \rmY\Theta = \begin{bmatrix}\rmX & \mathbf{0}_{n \times n}\end{bmatrix},\label{eq:sum-mat-sqrt-proof-step1}
\end{equation}
where $\rmX$ is an $n \times n$ lower triangular matrix. This implies that $\rmX$ is a square root factor of $\rmC$.

From \eqref{eq:sum-mat-sqrt-proof-step1}, we further have the following,
\begin{align}
    \rmY &= \begin{bmatrix}\rmX & \mathbf{0}_{n \times n}\end{bmatrix}\Theta^\top,\\
    \rmY^\top &= \Theta\begin{bmatrix}\rmX^\top \\ \mathbf{0}_{n \times n}\end{bmatrix},
\end{align}
where we post-multiply by $\Theta^\top$ in the first step and use the fact that $\Theta\Theta^\top = \rmI$, and transpose both sides in the second step. We have thus expressed $\rmY^\top$ as the product of an orthogonal matrix, $\Theta$, and an upper triangular matrix, $\begin{bmatrix}\rmX & \mathbf{0}_{n \times n}\end{bmatrix}^\top$. Such a factorization can be performed by $\mathrm{QR}$ decomposition. Thus, we can compute the square root factor $\rmC^{1/2} = \rmX$ via the $\mathrm{QR}$ decomposition of $\begin{bmatrix}\rmA^{1/2} & \rmB^{1/2}\end{bmatrix}^\top$. 
\end{proof}

\begin{algorithm}[ht]
\caption{Sum of Square Root Factors}\label{alg:sum-mat-sqrts}
\begin{algorithmic}[1]
\Function{SumMatrixSqrts}{$\rmA^{1/2}$, $\rmB^{1/2}$}
\State{$\underline{\hspace{1em}}, \begin{bmatrix}\rmC^{1/2} & \mathbf{0}_{n \times n}\end{bmatrix}^\top = \mathrm{QR}\left(\begin{bmatrix}\rmA^{1/2} & \rmB^{1/2}\end{bmatrix}^\top\right)$}
\State{\Return $\rmC^{1/2}$}
\EndFunction
\end{algorithmic}
\end{algorithm}

\section{Technical Details}
\label{app:implementation}
\subsection{Continuous-Discrete Bayesian Smoothing}
\label{app:smoothing}
Several approximate smoothing procedures based on Gaussian assumed density approximation have been proposed in the literature. We refer the reader to \citet{sarkka2013gaussian} for an excellent review of continuous-discrete smoothers. In the following, we discuss the \emph{Type II extended RTS smoother} which is linear in the smoothing solution. According to this smoother, the mean, $\rvm^s_t$, and covariance matrix, $\rmP^s_t$, of the Gaussian approximation to the smoothing density, $p_t(\rvz_t | \gY_T)$, follow the backward ODEs,
\begin{subequations}
\label{eq:smooth-linearization-approx}
\begin{align}
    \frac{d\rvm^s_t}{dt} &= \rvf(\rvm_t, t) + \rmC(\rvm_t, t)(\rvm^s_t - \rvm_t),\label{eq:mean-smooth}\\
    \frac{d\mathbf{P}^s_t}{dt} &= \rmC(\rvm_t, t)\rmP^s_t + \rmP^s_t\rmC^\top(\rvm_t, t) - \rmD(\rvm_t, t),\label{eq:cov-smooth}
\end{align}
\end{subequations}
where ($\rvm_t$, $\rmP_t$) is the filtering solution given by \eqref{eq:linearization-approx}, $\rmC(\rvm_t, t) = \rmF_{\rvz}(\rvm_t, t) + \rmD(\rvm_t, t)\mathbf{P}^{-1}_t$ and backward means that the ODEs are solved backwards in time from the filtering solution ($\rvm^s_T = \rvm_T$, $\rmP^s_T = \rmP_T$).
\subsection{Algorithms}
\label{app:algorithms}

\begin{algorithm}[hb]
\caption{Continuous-Discrete Bayesian Filtering}\label{alg:filtering}
\begin{algorithmic}[1]
\Function{Update}{$\rva_k$, $\rvm_k^-$, $(\rmP_k^-)^{1/2}$; $\rmH, \rmR$}
\State{$\rmR^{1/2} \gets \mathrm{cholesky}(\rmR)$}
\State{$\rmA \gets \begin{bmatrix}
    \rmR^{1/2} & \rmH(\rmP_k^{-})^{1/2}\\
    \mathbf{0}_{m \times d} & (\rmP_k^{-})^{1/2}
    \end{bmatrix}$}
\State{$\underline{\hspace{1em}}, \begin{bmatrix}
    \rmX & \mathbf{0}\\
    \rmY & \rmZ
    \end{bmatrix}^\top \gets \mathrm{QR}(\rmA^\top)$}
\State{$\rmK_k \gets \rmY\rmX^{-1}$}
\State{$\hat{\rva}_k \gets \rmH\rvm_k^-$}
\State{$\rvm_k \gets \rvm_k^- + \rmK_k(\rva_k - \hat{\rva}_k)$}
\State{$\rmP_k^{1/2} \gets \rmZ$}
\State{$\rmS_k^{1/2} \gets \rmX$}
\State{\Return $\rvm_k$, $\rmP_k^{1/2}$, $\hat{\rva}_k$, $\rmS_k^{1/2}$}
\EndFunction
\Statex
\Function{Predict}{$\rvm_k$, $\rmP_k^{1/2}$, $t_k$, $t_{k+1}$; $\rvf, \rmQ, \rmG$}
\State{$\bm{\Phi}_1 \gets \rmI$}
\State{$\{\tilde{\rvm}_j\}_{j=1}^n \gets \mathrm{odeint}\left(\frac{d\rvm_t}{dt} = \rvf(\rvm_t, t), \rvm_k, [\tau_1=t_k,\dots,\tau_n=t_{k+1}]\right)$}
\State{$\{\tilde{\bm{\Phi}}_j\}_{j=1}^n \gets \mathrm{odeint}\left(\frac{d\bm{\Phi}_{t}}{dt} = \rmF_{\rvz}(\rvm_t, t)\bm{\Phi}_{t}, \bm{\Phi}_1, [\tau_1=t_k,\dots,\tau_n=t_{k+1}]\right)$}
\LeftComment{the two coupled ODEs above are solved together.}
\State{$\rvm_{k+1}^{-} \gets \tilde{\rvm}_n$}
\State{$(\rmP_{k+1}^{-})^{1/2} \gets$ \Call{ReduceSumMatrixSqrts}{$\left[\tilde{\bm{\Phi}}_n\rmP_k^{1/2}, \sqrt{\frac{\eta}{2}}\tilde{\bm{\Phi}}_1\rmD_{\tau_1}^{1/2}, \sqrt{\eta}\tilde{\bm{\Phi}}_2\rmD_{\tau_2^{1/2}},\dots,\sqrt{\frac{\eta}{2}}\tilde{\bm{\Phi}}_n\rmD_{\tau_n}^{1/2}\right]$}}
\LeftComment{\textsc{ReduceSumMatrixSqrts} uses the \textsc{SumMatrixSqrts} function in Algorithm~\ref{alg:sum-mat-sqrts}, reducing it over the list.}
\State{\Return $\rvm_{k+1}^{-}$, $(\rmP_{k+1}^{-})^{1/2}$}
\EndFunction
\Statex
\Function{Filter}{$\rva_{0:T}$, $t_{0:T}$; $\theta$}
\State{$\bm{\mu}_0, \bm{\Sigma}_0, \rvf, \rmQ, \rmG, \rmH, \rmR \gets \theta$}
\State{$\rvm_0^-, (\rmP_0^-)^{1/2} \gets \bm{\mu}_0, \mathrm{cholesky}(\bm{\Sigma}_0)$}
\State{$\ell \gets 0$}
\For{$i \gets 0, T$}
    \State{$\rvm_i, \rmP_i^{1/2}, \hat{\rva}_i, \rmS_i^{1/2} \gets$ \Call{Update}{$\rva_i$, $\rvm_i^-$, $(\rmP_i^-)^{1/2}$; $\rmH, \rmR$}}
    \State{$\ell \gets \ell + \log \gN(\rva_i; \hat{\rva}_i, \rmS_i)$}
    \If{$i = T$}
        \State{\textbf{break}}
    \EndIf
    \State{$\rvm_{i+1}^-, (\rmP_{i+1}^-)^{1/2} \gets$ \Call{Predict}{$\rvm_i$, $\rmP_i^{1/2}$, $t_i$, $t_{i+1}$; $\rvf, \rmQ, \rmG$}}
\EndFor
\State{\Return $\{\rvm_i, \rmP_i^{1/2}\}_{i=0}^T$, $\ell$}
\EndFunction
\end{algorithmic}
\end{algorithm}
In this section, we discuss the \emph{stable} filtering and smoothing algorithms used in \ourmodel. We refer the reader to the accompanying code for specific implementation details. 

Algorithm~\ref{alg:sum-mat-sqrts} provides a utility function --- \textsc{SumMatrixSqrts} --- that uses Lemma~\ref{lemma:sum-sqrt-factors} to compute the square root factor of the sum of two matrices with square root factors. The square root factor version of the continuous-discrete Bayesian filtering algorithm is given in Algorithm~\ref{alg:filtering}. Note that the \textsc{Predict} step for linear time-invariant dynamics can be performed analytically using matrix exponentials~\citep[Ch.~6]{sarkka2019applied}. We used the analytic solver for some of our experiments. The Type II RTS smoothing algorithm (Algorithm~\ref{alg:smoothing}) takes the filtered distributions as input and computes the smoothed distribution at every filtered timestep. To compute the smoothed distribution between observed timesteps, we cache the filtered distributions at these timesteps and provide them to the \textsc{Smooth} function together with the filtered distributions at observed timesteps.

\begin{algorithm}[htb]
\caption{Continuous-Discrete Type II Extended RTS Smoothing}\label{alg:smoothing}
\begin{algorithmic}[1]
\Function{SmoothStep}{$\rvm_k^s$, $(\rmP_k^s)^{1/2}$, $\rvm_k$, $\rmP_k^{1/2}$, $t_k$, $t_{k-1}$; $\rvf, \rmQ, \rmG$}
\State{$\bm{\Phi}^s_1 \gets \rmI$}
\State{$\{\tilde{\rvm}^s_j\}_{j=1}^n \gets \mathrm{odeint}\left(\frac{d\rvm^s_t}{dt} = \rvf(\rvm_k, t) + \rmC(\rvm_k, t)(\rvm^s_t - \rvm_k), \rvm^s_k, [\tau_1=t_k,\dots,\tau_n=t_{k-1}]\right)$}
\State{$\{\tilde{\bm{\Phi}}^s_j\}_{j=1}^n \gets \mathrm{odeint}\left(\frac{d\bm{\Phi}^s_{t}}{dt} = \rmC(\rvm_k, t)\bm{\Phi}^s_{t}, \bm{\Phi}^s_1, [\tau_1=t_k,\dots,\tau_n=t_{k-1}]\right)$}
\LeftComment{the two coupled ODEs above are solved together.}
\State{$\rvm_{k-1}^{s} \gets \tilde{\rvm}^s_n$}
\State{$(\rmP_{k-1}^{s})^{1/2} \gets$ \Call{ReduceSumMatrixSqrts}{$\left[\tilde{\bm{\Phi}}^s_n\rmP_k^{1/2}, \sqrt{\frac{\eta}{2}}\tilde{\bm{\Phi}}^s_1\rmD_{\tau_1}^{1/2}, \sqrt{\eta}\tilde{\bm{\Phi}}^s_2\rmD_{\tau_2^{1/2}},\dots,\sqrt{\frac{\eta}{2}}\tilde{\bm{\Phi}}^s_n\rmD_{\tau_n}^{1/2}\right]$}}
\LeftComment{\textsc{ReduceSumMatrixSqrts} uses the \textsc{SumMatrixSqrts} function in Algorithm~\ref{alg:sum-mat-sqrts}, reducing it over the list.}
\State{\Return $\rvm_{k-1}^{s}$, $(\rmP_{k-1}^{s})^{1/2}$}
\EndFunction
\Statex
\Function{Smooth}{$\{\rvm_i, \rmP_i^{1/2}\}_{i=0}^T$, $t_{0:T}$; $\theta$}
\State{$\bm{\mu}_0, \bm{\Sigma}_0, \rvf, \rmQ, \rmG, \rmH, \rmR \gets \theta$}
\State{$\rvm_T^s, (\rmP_T^s)^{1/2} \gets \rvm_T, \rmP_T^{1/2}$}
\For{$i \gets T, 1$}\Comment{note the time reversal.}
    \State{$\rvm_{i-1}^s, (\rmP_{i-1}^s)^{1/2} \gets$ \Call{SmoothStep}{$\rvm_i^s$, $(\rmP_i^s)^{1/2}$, $\rvm_{i-1}$, $\rmP_{i-1}^{1/2}$, $t_i$, $t_{i-1}$; $\rvf, \rmQ, \rmG$}}
\EndFor
\State{\Return $\{\rvm_i^s, (\rmP_i^s)^{1/2}\}_{i=0}^T$}
\EndFunction
\end{algorithmic}
\end{algorithm}
\subsection{Stable Implementation (Contd.)}
\label{app:stable-implementation}
\paragraph{Square Root Factor Measurement Update.} In Section~\ref{sec:stable-implementation}, we discussed a square root factor version of the measurement update step via the $\mathrm{QR}$ decomposition of $\rmA^\top$, where,
\begin{equation}
    \rmA = \begin{bmatrix}
    \rmR^{1/2} & \rmH(\rmP_k^{-})^{1/2}\\
    \mathbf{0}_{m \times d} & (\rmP_k^{-})^{1/2}
    \end{bmatrix}.
\end{equation}
Let
$
    \Theta, \rmU = \mathrm{QR}(\rmA^\top)
$, where
\begin{equation}
    \rmU = \begin{bmatrix}
    \rmX & \mathbf{0}\\
    \rmY & \rmZ
    \end{bmatrix}^\top.
\end{equation}
In the following, we show how $\rmP_k^{1/2} = \rmZ$. Our proof is based on \citet{zonov2019kalman} and we refer the reader to \citet[Appendix~A]{zonov2019kalman} for the proof of $\rmK_k = \rmY\rmX^{-1}$.
\begin{proof}
Note that $\rmA$ is a square root factor of 
\begin{equation}
    \begin{bmatrix}
    \rmR + \rmH\rmP_k^{-}\rmH^\top & \rmH\rmP_k^{-}\\
    (\rmP_k^{-})^\top\rmH^\top & \rmP_k^{-}
    \end{bmatrix}.\label{eq:sqrt-update-step-1}
\end{equation}
Matching the terms in (\ref{eq:sqrt-update-step-1}) with the terms in
\begin{equation}
    \rmU\rmU^\top = \begin{bmatrix}
    \rmX\rmX^\top & \rmX\rmY^\top \\
    \rmY\rmX^\top & \rmY\rmY^\top + \rmZ\rmZ^\top
    \end{bmatrix},
\end{equation}
we get the following equations,
\begin{subequations}
    \begin{align}
    \rmX\rmX^\top &= \rmR + \rmH\rmP_k^{-}\rmH^\top,\\
    \rmX\rmY^\top &= \rmH\rmP_k^{-},\\
    \rmY\rmX^\top &= (\rmP_k^{-})^\top\rmH^\top,\\
    \rmY\rmY^\top + \rmZ\rmZ^\top &= \rmP_k^{-}.\label{eq:sqrt-update-step-2-d}
\end{align}
\label{eq:sqrt-update-step-2}
\end{subequations}
From \eqref{eq:sqrt-update-step-2-d}, we have the following,
\begin{subequations}
    \begin{align}
    \rmY\rmY^\top + \rmZ\rmZ^\top &= \rmP_k^{-},\\
    \rmZ\rmZ^\top &= \rmP_k^{-} - \rmY\rmY^\top,\\
    \rmZ\rmZ^\top &= \rmP_k^{-} - \rmY(\rmX^\top\rmX^{-\top})(\rmX^{-1}\rmX)\rmY^\top,\\
    \rmZ\rmZ^\top &= \rmP_k^{-} - \rmY\rmX^\top(\rmX\rmX^{\top})^{-1}\rmX\rmY^\top,
\end{align}
\end{subequations}
where we introduce $\rmI = (\rmX^\top\rmX^{-\top})(\rmX^{-1}\rmX)$ in the third step and use the property $(\rmX\rmX^{\top})^{-1} = \rmX^{-\top}\rmX^{-1}$ in the last step. Substituting values from \eqref{eq:sqrt-update-step-2}, we get,
\begin{subequations}
    \begin{align}
    \rmZ\rmZ^\top &= \rmP_k^{-} - (\rmP_k^{-})^\top\rmH^\top\rmS_k^{-1}\rmH\rmP_k^{-}\\
    \rmZ\rmZ^\top &= \rmP_k^{-} - (\rmP_k^{-})^\top\rmH^\top\rmS_k^{-1}(\rmS_k\rmS_k^{-1})\rmH\rmP_k^{-}\\
    \rmZ\rmZ^\top &= \rmP_k^{-} - \rmP_k^{-}\rmH^\top\rmS_k^{-1}\rmS_k\rmS_k^{-\top}\rmH(\rmP_k^{-})^\top\\
    \rmZ\rmZ^\top &= \rmP_k^{-} - \rmK_k\rmS_k\rmK_k^\top
\end{align}
\end{subequations}
where we introduce $\rmI = \rmS_k\rmS_k^{-1}$ in the second step, use the fact that $\rmS_k$ is symmetric in the third step and substitute the value of $\rmK_k$ from \eqref{eq:kalman-gain} in last step. Note that
$
\rmZ\rmZ^\top = \rmP_k^{-} - \rmK_k\rmS_k\rmK_k^\top  = \rmP_k$; therefore, $\rmZ = \rmP_k^{1/2}$. 
\end{proof}

\paragraph{Regularizing Non-Linear Dynamics.} We now discuss the techniques we employed to regularize the latent dynamics in \ourmodel. Particularly in the case of non-linear dynamics (\ctssmnl), regularization is critical for stable training. The drift function, $\rvf$, was parameterized by an MLP in all our experiments. We experimented with the $\tanh$ and $\mathrm{softplus}$ non-linearities. We found that applying the non-linearity after the last layer was important when using $\tanh$. Furthermore, we also initialized the parameters of the last layer to $0$ when using $\tanh$. In the case of experiments with a large time interval (e.g., MoCap and USHCN), application of spectral normalization~\citep{miyato2018spectral} along with the $\mathrm{softplus}$ non-linearity proved critical for stable training. In the following, we present our hypothesis on why spectral normalization stabilizes training.

According to \citet[Section~5.2]{oksendal2003stochastic}, one of the conditions for the existence of a unique solution of an SDE is the Lipschitz continuity of the drift function, $\rvf$. Applying spectral normalization regularizes the neural network to be 1-Lipschitz, aiding its solvability using numerical methods. However, spectral normalization is even more important in the case of \ourmodel\ from a practical perspective --- it prevents the numerical explosion of the elements of $\bm{\Phi}_t$ in the prediction step \eqrefp{eq:fundamental-ode}, as discussed below. 

Consider the case of a fixed Jacobian matrix $\rmF_\rvz$ in an interval $[t_1, t_2]$. In this case, the solution of \eqref{eq:fundamental-ode} is given by
\begin{equation}
    \bm{\Phi}_{t_2} = \exp{(\rmF_\rvz(t_2 - t_1))}\bm{\Phi}_{t_1},
\end{equation}
where $\exp{(\rmF_\rvz(t_2 - t_1))}$ denotes the matrix exponential. For unregularized drifts, the elements of $\exp{(\rmF_\rvz(t_2 - t_1))}$ can become arbitrarily large. However, in the case of 1-Lipschitz drift functions (as provided by spectral normalization), the spectral norm of $\exp{(\rmF_\rvz)}$ is bounded by $\exp(1)$, as shown in Lemma~\ref{lemma:lip-mat-exp}. This controls the growth rate of the elements of fundamental matrix, $\bm{\Phi}_t$. 
\begin{lemma}
\label{lemma:lip-mat-exp}
Let $\rvg: \R^m \to \R^m$ be a 1-Lipschitz function and $\rmJ_{\rvg}: \R^m \to \R^{m \times m}$ be its Jacobian function. Then, $\|\exp(\rmJ_{\rvg}(\rvz))\|_2 \leq \exp(1)~\forall~\rvz \in \R^m$ where $\|\cdot\|_2$ denotes the spectral norm of a matrix.
\end{lemma}
\begin{proof}
    The spectral norm of the Jacobian of a $K$-Lipschitz function is bounded by $K$. Thus, we have,
    \begin{equation}
        \|\rmJ_{\rvg}(\rvz)\|_2 \leq 1~\forall~\rvz \in \R^m.\label{eq:lip-jac-bound}
    \end{equation}
    Using the power series representation of the matrix exponential,
    $$
    \exp(\rmA) = \sum_{k=0}^\infty \frac{\rmA^k}{k!},
    $$
    we get the following bound on $\|\exp(\rmJ_{\rvg}(\rvz))\|_2$, 
    \begin{equation}
        \|\exp(\rmJ_{\rvg}(\rvz))\|_2 \leq \sum_{k=0}^\infty \left\|\frac{\rmJ_{\rvg}(\rvz)^k}{k!}\right\|_2 \leq \sum_{k=0}^\infty \frac{\left\|\rmJ_{\rvg}(\rvz)\right\|_2^k}{k!} = \exp(\left\|\rmJ_{\rvg}(\rvz)\right\|_2).\label{eq:mat-exp-spectral-bound}
    \end{equation}
    Combining \eqref{eq:mat-exp-spectral-bound} with \eqref{eq:lip-jac-bound}, we get,
    \begin{equation}
        \|\exp(\rmJ_{\rvg}(\rvz))\|_2 \leq \exp(\left\|\rmJ_{\rvg}(\rvz)\right\|_2) \leq \exp(1),
    \end{equation}
    which completes the proof.
\end{proof}

For the same reasons as discussed above, we initialized the transition matrices in our linear models to be random orthogonal matrices as orthogonal matrices have spectral norm equal to 1. However, in the case of \ctssmlti\ on the USHCN dataset, this initialization was not sufficient during the initial phase of training and we used a random skew-symmetric matrix instead. The matrix exponential of a skew-symmetric matrix is an orthogonal matrix. We generated a random skew-symmetric matrix as follows,
\begin{align*}
    \rmF &\sim \left[\gN(0, 1)\right]^{m \times m},\\
    \rmF &= \left(\frac{\rmF - \rmF^\top}{2}\right).
\end{align*}

\paragraph{Fixed Measurement Matrix.} We used a fixed rectangular identity matrix as the auxiliary measurement matrix ($\rmH$ in Eq.~\ref{eq:auxiliary-emission}) in our bouncing ball, damped pendulum and CMU MoCap (walking) experiments as it lead to improved learning of dynamics. This parameterization forces the model to learn the static (e.g., position) and dynamic (e.g., velocity) components in separate elements of the latent state, thereby disentangling them~\citep{klushyn2021latent}. 
\subsection{Imputation and Forecasting}
In this section, we describe how to perform imputation and forecasting using a trained \ourmodel. 

For imputation, the timesteps at which imputation is to be performed are provided to the \textsc{Filter} function during filtering. The filtered distributions are then passed to the \textsc{Smooth} function and (imputed) samples are drawn from the smoothed distributions. 

For forecasting, filtering is first performed over the context time series. The \textsc{Predict} function is then used up to end of the forecast horizon, starting from the last filtered distribution. Sample forecast trajectories are then drawn from these predicted distributions.

\section{Experiment Details}
\label{app:exp-details}

\subsection{Datasets}
\label{app:datasets}

\paragraph{Bouncing Ball and Damped Pendulum.} 
The bouncing ball dataset comprises univariate time series of the position of a ball bouncing between two fixed walls, in the absence of dissipative forces. 
The initial position, $x_0$, and velocity, $v_0$, of the ball are chosen at random, as follows,
\begin{align}
    x_0 &\sim \gU(-1, 1),\\
    v_0 &\sim \gU(0.05, 0.5) \times \gU\{-1, 1\},
\end{align}
where $\gU(a, b)$ denotes a uniform distribution on $(a, b)$ and $\gU\{c_1, \dots, c_k\}$ denotes a uniform categorical distribution on $\{c_1, \dots, c_k\}$. The observed position, $y_k$, is a corrupted version of the true position, $x_k$,
\begin{align}
    y_k \sim \gN(x_k, 0.05^2).
\end{align}
Collisions with the walls, located at $-1$ and $+1$, are assumed to be perfectly elastic, i.e., the sign of the velocity gets flipped when the ball hits either of the walls. Thus, the ball exhibits piecewise-linear dynamics. We used the Euler integrator with a step size of 0.1s to simulate the dynamics. The training, validation, and test datasets consist of 5000, 500, and 500 sequences of length 30s each, respectively. 

The damped pendulum dataset~\citep{karl2016deep,kurle2020deep} comprises bivariate time series of the XY-coordinates of a pendulum oscillating in the presence of a damping force.
The non-linear latent dynamics of this dataset is given by,
\begin{align}
    \frac{d\theta_t}{dt} &= \omega_t,\\
    \frac{d\omega_t}{dt} &= -\frac{g}{l}\sin(\theta_t)  -\frac{\gamma}{m}\omega_t,
\end{align}
where $\theta_t$ and $\omega_t$ are the angle and angular velocity, respectively, and $g=9.81$, $l=1$, $m=1$, and $\gamma=0.25$ are the acceleration due to gravity, the length of the massless cord of the pendulum, the mass of the pendulum bob, and the damping coefficient, respectively. 
The initial angle, $\theta_0$, and angular velocity, $\omega_0$, of the pendulum are chosen at random, as follows,
\begin{align}
    \theta_0 &= \pi + \mathrm{clip}\left(\eps, -2, 2\right),\\
    \omega_0 &= 4\times\mathrm{clip}\left(\eps, -2, 2\right),
\end{align}
where $\eps \sim \gN(0, 1)$ and $\mathrm{clip}(x, a, b)$ denotes clipping the value of $x$ between $a$ and $b$. The observations are Cartesian coordinates of the pendulum's bob with additive Gaussian noise, $\gN(0, 0.05^2)$. We used the RK4 integrator to simulate the latent dynamics with a step size of 0.1s. The training, validation, and test datasets consist of 5000, 1000, and 1000 sequences of length 15s each, respectively. 
\paragraph{CMU Motion Capture (Walking).} The CMU Motion Capture database\footnote{The original CMU MoCap database is available at: \url{http://mocap.cs.cmu.edu}.} comprises time series of joint angles of human subjects performing everyday activities, e.g., walking, running, and dancing. We used walking sequences of subject 35 from this database for our experiments. A preprocessed version of this dataset from \citet{yildiz2019ode2vae} consists of 23 50-dimensional sequences of 300 timesteps each, split into 16 training, 3 validation and 4 test sequences. 
\paragraph{USHCN Climate Indicators.} The USHCN Climate dataset\footnote{The original USHCN Climate dataset is available at: \url{https://cdiac.ess-dive.lbl.gov/ftp/ushcn_daily/}.} consists of measurements of five climate indicators --- precipitation, snowfall, snow depth, minimum temperature, and maximum temperature --- across the United States. The preprocessed version of this dataset from \citet{de2019gru} contains sporadic time series from 1,114 meteorological stations with a total of 386,068 unique observations over 4 years, between 1996 and 2000. The timestamps are scaled to lie in $[0, 200]$. The 1,114 stations are split into 5 folds of 70\% training, 20\% validation, and 10\% test stations, respectively. 
\paragraph{Pymunk Physical Environments.} The Pymunk physical environments datasets are video datasets of physical environments simulated using the Pymunk Physics engine. We used two environments proposed in \citet{fraccaro2017disentangled}:  Box and Pong. Each frame of these videos is a 32 $\times$ 32 binary image. The Box dataset consists of videos of a ball moving inside a 2-dimensional box with perfectly elastic collisions with the walls of the box. The Pong dataset consists of videos of a Pong-like environment with a ball and two paddles that move to keep the ball inside the frame. Both datasets consist of 5000 training, 100 validation, and 1000 test videos with 60 frames each. We refer the reader to \citet{fraccaro2017disentangled} for further details on how these datasets are generated\footnote{The scripts for generating Pymunk datasets are available at: \url{https://github.com/simonkamronn/kvae}.}. 

\subsection{Training and Evaluation Setups}
\paragraph{Bouncing Ball and Damped Pendulum.} We trained all the models on the first 10s/5s of the sequences (i.e., 100/50 steps) from the training dataset for the bouncing ball/damped pendulum datasets. We randomly dropped 30\%, 50\%, and 80\% of the training steps for the missing-data experiments. For evaluation, we report the MSE over the missing (for imputation) and the next 200/100 timesteps (for forecast) for the bouncing ball/damped pendulum test datasets. The MSE was averaged over 5 independent runs for 50 sample trajectories.

\paragraph{CMU Motion Capture (Walking).} For Setup~1, we trained \ourmodel\ models on complete 300-timestep sequences from the training set. During test time, we evaluated the predictive performance on the next 297 steps with a context of the first 3 steps from the test set. For Setup~2, we trained the models on the first 200 timesteps from sequences in the training set. During test time, we provided the models with a context of the first 100 timesteps from sequences in the test set and evaluated their performance on the next 200 timesteps. We report the MSE averaged over 50 sample trajectories together with 95\% prediction interval based on the $t$-statistic for a single run, as reported in prior works.

\paragraph{USHCN Climate Indicators.} We trained \ourmodel\ models under the same setup as \citet{de2019gru} using 4 years of observations from the training stations. During test time, we provided the models with the first 3 years of observations from the test set as context and evaluated their performance on the accuracy of the next 3 measurements. The MSE was computed between the mean of 50 sample forecast trajectories (simulating a point forecast) and the ground truth, averaged over the 5 folds.  

\paragraph{Pymunk Physical Environments.} We trained the models on the first 20 frames of the videos from the training dataset with 20\% of the frames randomly dropped. During test time, we provided the models with a context of 20 frames and evaluated the forecast performance on the next 40 frames. We report the EMD between the predicted and the ground truth frames, averaged over 16 sample trajectories. The EMD was computed using the \texttt{ot.emd2} function from the Python Optimal Transport (POT) library~\citep{flamary2021pot} with the \texttt{euclidean} metric as the cost function. 

\subsection{Experiment Configurations}
We ran all our experiments on 2 machines with 1 Tesla T4 GPU, 16 CPUs, and 64 GB of memory each. In this section, we report training and hyperparameter configurations used in our experiments. We refer the reader to the accompanying code for specific details.

We optimized all models using the Adam optimizer with a learning rate of 0.01 for all the datasets except Pymunk physical environments where we used 0.002. We reduced the learning rate exponentially with a decay rate of 0.9 every 500 steps for the bouncing ball, damped pendulum, and CMU MoCap (walking) datasets, every 100 steps for the USHCN climate dataset, and every 3000 steps for the Pymunk physical environments datasets. We trained the models for 5K, 2K, 2.5K, 150, and 100K steps with a batch size of 50, 64, 16, 100, and 32 for the bouncing ball, damped pendulum, CMU MoCap (walking), USHCN climate indicators, and Pymunk physical environments, respectively. 

For \ourmodel\ models, we used the following auxiliary inference and emission networks for each dataset:
\begin{itemize}[noitemsep,nolistsep]
    \item \textbf{Bouncing Ball, Damped Pendulum, and USHCN Climate Indicators}
    \begin{itemize}[noitemsep,nolistsep]
        \item Auxiliary inference network: \texttt{Identity()}
        \item Emission network: \texttt{Identity()}
    \end{itemize}
    \item \textbf{CMU Motion Capture (Walking)}
    \begin{itemize}[noitemsep,nolistsep]
        \item Auxiliary inference network: \texttt{Input(d) $\rightarrow$ Linear(64) $\rightarrow$ Softplus() $\rightarrow$ Linear (2$\times$h)}
	\item Emission network: \texttt{Input(h) $\rightarrow$ 2$\times$[Linear(30) $\rightarrow$ Softplus()] $\rightarrow$ Linear (d)}
    \end{itemize}
    \item \textbf{Pymunk Physical Environments}
    \begin{itemize}[noitemsep,nolistsep]
        \item Auxiliary inference network: \texttt{Input(1, 32, 32) $\rightarrow$ ZeroPad2d(padding=[0, 1, 0, 1]) $\rightarrow$ Conv2d(1, 32, kernel\_size=3, stride=2) $\rightarrow$ ReLU() $\rightarrow$ 2$\times$[ZeroPad2d(padding=[0, 1, 0, 1]) $\rightarrow$ Conv2d(32, 32, kernel\_size=3, stride=2) $\rightarrow$ ReLU()] $\rightarrow$ Flatten $\rightarrow$ Linear(64) $\rightarrow$ Linear(2$\times$h)}
	\item Emission network: \texttt{Input(h) $\rightarrow$ Linear(512) $\rightarrow$ 3$\times$[Conv2d(32, 128, kernel\_size=3, stride=1, padding=1) $\rightarrow$ ReLU() $\rightarrow$ PixelShuffle(upscale\_factor=2)] $\rightarrow$  Conv2d(32, 1, kernel\_size=1, stride=1)}
    \end{itemize}
\end{itemize}

To ensure good initial estimation of auxiliary variables, we did not update the underlying SSM parameters for the first 100 and 1000 training steps for the CMU MoCap (walking) and Pymunk physical environments datasets, respectively. In the following, we list specific experiment configurations for individual experiments.

\subsubsection{LatentODE}
We used the RK4 ODE solver to integrate the encoder and drift ODEs with a step size of 0.05 for all datasets.

\begin{itemize}[noitemsep,nolistsep]
\item\textbf{Bouncing Ball}
\begin{itemize}[noitemsep,nolistsep]
	\item Dimension of latent state: 6
	\item Dimension of observations: 1
	\item Encoder network: ODEGRU with a \texttt{GRUCell(hidden\_units=10)} and ODE drift function \texttt{Input(10) $\rightarrow$ Linear(30) $\rightarrow$ Tanh() $\rightarrow$ Linear(10)}
	\item Decoder network: \texttt{Input(6) $\rightarrow$ Linear(10) $\rightarrow$ Softplus() $\rightarrow$ Linear(1)}
	\item ODE drift function: \texttt{Input(6) $\rightarrow$ Linear(64) $\rightarrow$ Softplus() $\rightarrow$ Linear(6)}

\end{itemize}

\item\textbf{Damped Pendulum}
\begin{itemize}[noitemsep,nolistsep]
	\item Dimension of latent state: 6
	\item Dimension of observations: 2
	\item Encoder network: ODEGRU with a \texttt{GRUCell(hidden\_units=10)} and ODE drift function \texttt{Input(10) $\rightarrow$ Linear(64) $\rightarrow$ Tanh() $\rightarrow$ Linear(10)}
	\item Decoder network: \texttt{Input(6) $\rightarrow$ Linear(64) $\rightarrow$ Tanh() $\rightarrow$ Linear(2)}
	\item ODE drift function: \texttt{Input(6) $\rightarrow$ Linear(64) $\rightarrow$ Tanh() $\rightarrow$ Linear(6)}

\end{itemize}

\item\textbf{CMU Motion Capture (Walking)}
\begin{itemize}[noitemsep,nolistsep]
	\item Dimension of latent state: 10
	\item Dimension of observations: 50
	\item Encoder network: ODEGRU with a \texttt{GRUCell(hidden\_units=30)} and ODE drift function \texttt{Input(30) $\rightarrow$ Linear(64) $\rightarrow$ Tanh() $\rightarrow$ Linear(30)}
	\item Decoder network: \texttt{Input(10) $\rightarrow$ 2$\times$[Linear(30) $\rightarrow$ Softplus()] $\rightarrow$ Linear(50)}
	\item ODE drift function: \texttt{Input(10) $\rightarrow$ Linear(30) $\rightarrow$ Softplus() $\rightarrow$ Linear(10)}

\end{itemize}

\item\textbf{Pymunk Physical Environments}
\begin{itemize}[noitemsep,nolistsep]
	\item Dimension of latent state: 10
	\item Dimension of observations: 1024
	\item Encoder network: Same CNN encoder base as in the auxiliary inference network in \ourmodel\ models and ODEGRU with a \texttt{GRUCell(hidden\_units=64)} and ODE drift function  \texttt{Input(64) $\rightarrow$ Linear(64) $\rightarrow$ Tanh() $\rightarrow$ Linear(64)}
	\item Decoder network: Same CNN decoder as in the emission network in \ourmodel\ models
	\item ODE drift function: \texttt{Input(10) $\rightarrow$ Linear(64) $\rightarrow$ Tanh() $\rightarrow$ Linear(10)}

\end{itemize}

\end{itemize}

\subsubsection{LatentSDE}
For LatentSDE experiments, we additionally annealed the KL term in the objective function with a linear annealing schedule from 0 to 1 over 500 steps for all datasets except Pymunk physical environments for which we annealed over 1000 steps. As proposed in \citet{li2020scalable}, we also provided the posterior SDEs with an additional context vector from the encoder to incorporate information from later observations. We used the RK4 ODE solver to integrate the encoder ODEs and the Euler-Maruyama SDE solver to integrate the prior/posterior SDEs with a step size of 0.05 for all datasets.

\begin{itemize}[noitemsep,nolistsep]
    \item\textbf{Bouncing Ball}
\begin{itemize}[noitemsep,nolistsep]
	
	\item Dimension of latent state: 6
	\item Dimension of context vector: 3
	\item Dimension of observations: 1
	\item Encoder network: ODEGRU with a \texttt{GRUCell(hidden\_units=10)} and ODE drift function  \texttt{Input(10) $\rightarrow$ Linear(64) $\rightarrow$ Tanh() $\rightarrow$ Linear(10)}
	\item Decoder network: \texttt{Input(6) $\rightarrow$ Linear(64) $\rightarrow$ Softplus() $\rightarrow$ Linear(1)}
	\item Posterior SDE drift function: \texttt{Input(6+3) $\rightarrow$ Linear(64) $\rightarrow$ Softplus() $\rightarrow$ Linear(6)}
	\item Prior SDE drift function: \texttt{Input(6) $\rightarrow$ Linear(64) $\rightarrow$ Softplus() $\rightarrow$ Linear(6)}
	\item Posterior/Prior SDE diffusion function: \texttt{6$\times$[Input(1) $\rightarrow$ Linear(64) $\rightarrow$ Softplus() $\rightarrow$ Linear(1)]}

\end{itemize}

\item\textbf{Damped Pendulum}
\begin{itemize}[noitemsep,nolistsep]
	
	\item Dimension of latent state: 6
	\item Dimension of context vector: 3
	\item Dimension of observations: 2
	\item Encoder network: ODEGRU with a \texttt{GRUCell(hidden\_units=10)} and ODE drift function  \texttt{Input(10) $\rightarrow$ Linear(64) $\rightarrow$ Tanh() $\rightarrow$ Linear(10)}
	\item Decoder network: \texttt{Input(6) $\rightarrow$ Linear(64) $\rightarrow$ Tanh() $\rightarrow$ Linear(2)}
	\item Posterior SDE drift function: \texttt{Input(6+3) $\rightarrow$ Linear(64) $\rightarrow$ Softplus() $\rightarrow$ Linear(6)}
	\item Prior SDE drift function: \texttt{Input(6) $\rightarrow$ Linear(64) $\rightarrow$ Softplus() $\rightarrow$ Linear(6)}
	\item Posterior/Prior SDE diffusion function: \texttt{6$\times$[Input(1) $\rightarrow$ Linear(64) $\rightarrow$ Softplus() $\rightarrow$ Linear(1)]}

\end{itemize}

\item\textbf{CMU Motion Capture (Walking)}
\begin{itemize}[noitemsep,nolistsep]
	
	\item Dimension of latent state: 10
	\item Dimension of context vector: 3
	\item Dimension of observations: 50
	\item Encoder network: ODEGRU with a \texttt{GRUCell(hidden\_units=30)} and ODE drift function  \texttt{Input(30) $\rightarrow$ Linear(64) $\rightarrow$ Tanh() $\rightarrow$ Linear(30)}
	\item Decoder network: \texttt{Input(10) $\rightarrow$ 2$\times$[Linear(30) $\rightarrow$ Softplus()] $\rightarrow$ Linear(50)}
	\item Posterior SDE drift function: \texttt{Input(10+3) $\rightarrow$ Linear(30) $\rightarrow$ Softplus() $\rightarrow$ Linear(10)}
	\item Prior SDE drift function: \texttt{Input(10) $\rightarrow$ Linear(30) $\rightarrow$ Softplus() $\rightarrow$ Linear(10)}
	\item Posterior/Prior SDE diffusion function: \texttt{10$\times$[Input(1) $\rightarrow$ Linear(30) $\rightarrow$ Softplus() $\rightarrow$ Linear(1)]}

\end{itemize}

\item\textbf{Pymunk Physical Environments}
\begin{itemize}[noitemsep,nolistsep]
	
	\item Dimension of latent state: 10
	\item Dimension of context vector: 4
	\item Dimension of observations: 1024
	\item Encoder network: Same CNN encoder base as in the auxiliary inference network in \ourmodel\ models and ODEGRU with a \texttt{GRUCell(hidden\_units=64)} and ODE drift function  \texttt{Input(64) $\rightarrow$ Linear(64) $\rightarrow$ Tanh() $\rightarrow$ Linear(64)}
	\item Decoder network: Same CNN decoder as in the emission network in \ourmodel\ models
	\item Posterior SDE drift function: \texttt{Input(10+4) $\rightarrow$ Linear(64) $\rightarrow$ Tanh() $\rightarrow$ Linear(10) $\rightarrow$ Tanh()}
	\item Prior SDE drift function: \texttt{Input(10) $\rightarrow$ Linear(64) $\rightarrow$ Tanh() $\rightarrow$ Linear(10) $\rightarrow$ Tanh()}
	\item Posterior/Prior SDE diffusion function: \texttt{10$\times$[Input(1) $\rightarrow$ Linear(64) $\rightarrow$ Softplus() $\rightarrow$ Linear(1)]}

\end{itemize}

\end{itemize}

\subsubsection{\ctssmlti}
\begin{itemize}[noitemsep,nolistsep]
    
\item \textbf{Bouncing Ball}
\begin{itemize}[noitemsep,nolistsep]

	\item Dimension of state ($m$): 6
	\item Dimension of auxiliary variables ($h$): 1
	\item Dimension of observations ($d$): 1

	\item Integrator: Analytic
\end{itemize}

\item \textbf{Damped Pendulum}
\begin{itemize}[noitemsep,nolistsep]

	\item Dimension of state ($m$): 6
	\item Dimension of auxiliary variables ($h$): 2
	\item Dimension of observations ($d$): 2

	\item Integrator: Analytic
\end{itemize}

\item \textbf{CMU Motion Capture (Walking)}
\begin{itemize}[noitemsep,nolistsep]

	\item Dimension of state ($m$): 10
	\item Dimension of auxiliary variables ($h$): 6
	\item Dimension of observations ($d$): 50
	
	\item Integrator: Analytic
\end{itemize}

\item \textbf{USHCN Climate Indicators}
\begin{itemize}[noitemsep,nolistsep]

	\item Dimension of state ($m$): 10
	\item Dimension of auxiliary variables ($h$): 5
	\item Dimension of observations ($d$): 5
	
	\item Integrator: Euler with step size 0.1
\end{itemize}

\item \textbf{Pymunk Physical Environments}
\begin{itemize}[noitemsep,nolistsep]

	\item Dimension of state ($m$): 10
	\item Dimension of auxiliary variables ($h$): 4
	\item Dimension of observations ($d$): 1024
	
	\item Integrator: RK4 with step size 0.05
\end{itemize}
\end{itemize}

\subsubsection{\ctssmnl}
We set the diffusion function to $\rmG(\cdot, t) = \rmI$ for all datasets.

\begin{itemize}[noitemsep,nolistsep]
\item\textbf{Bouncing Ball}
\begin{itemize}[noitemsep,nolistsep]

    \item Dimension of state ($m$): 6
    \item Dimension of auxiliary variables ($h$): 1
    \item Dimension of observations ($d$): 1
    \item Drift function ($\rvf$): \texttt{Input(m) $\rightarrow$ Linear(64) $\rightarrow$ Softplus() $\rightarrow$ Linear(m)}

    \item Integrator: RK4 with step size 0.05
\end{itemize}

\item\textbf{Damped Pendulum}
\begin{itemize}[noitemsep,nolistsep]

    \item Dimension of state ($m$): 6
    \item Dimension of auxiliary variables ($h$): 2
    \item Dimension of observations ($d$): 2
    \item Drift function ($\rvf$): \texttt{Input(m) $\rightarrow$ Linear(64) $\rightarrow$ Softplus() $\rightarrow$ Linear(m)}

    \item Integrator: RK4 with step size 0.05
\end{itemize}

\item\textbf{CMU Motion Capture (Walking)}
\begin{itemize}[noitemsep,nolistsep]

    \item Dimension of state ($m$): 10
    \item Dimension of auxiliary variables ($h$): 6
    \item Dimension of observations ($d$): 50
    \item Drift function ($\rvf$):	\texttt{Input(m) -> SN(Linear(30)) -> Softplus() -> SN(Linear(m))}

    \item Integrator: RK4 with step size 0.05
\end{itemize}

\item\textbf{USHCN Climate Indicators}
\begin{itemize}[noitemsep,nolistsep]

    \item Dimension of state ($m$): 10
    \item Dimension of auxiliary variables ($h$): 5
    \item Dimension of observations ($d$): 5
    \item Drift function ($\rvf$): \texttt{Input(m) $\rightarrow$ SN(Linear(64)) $\rightarrow$ Softplus() $\rightarrow$ SN(Linear(m))}

    \item Integrator: Euler with step size 0.1
\end{itemize}

\item\textbf{Pymunk Physical Environments}
\begin{itemize}[noitemsep,nolistsep]

    \item Dimension of state ($m$): 10
    \item Dimension of auxiliary variables ($h$): 4
    \item Dimension of observations ($d$): 1024
    \item Drift function ($\rvf$): \texttt{Input(m) $\rightarrow$ Linear(64) $\rightarrow$ Tanh() $\rightarrow$ Linear(m) $\rightarrow$ Tanh()}

    \item Integrator: RK4 with step size 0.05
\end{itemize}
\end{itemize}
\subsubsection{\ctssmll}
We set the $\alpha$-network to \texttt{Input(m) $\rightarrow$ Linear(64) $\rightarrow$ Softplus() $\rightarrow$ Linear(K)} for all datasets.
\begin{itemize}[noitemsep,nolistsep]
\item\textbf{Bouncing Ball}
\begin{itemize}[noitemsep,nolistsep]

	\item Dimension of state ($m$): 6
	\item Dimension of auxiliary variables ($h$): 1
	\item Dimension of observations ($d$): 1
	\item Number of base matrices ($K$): 5

	\item Integrator: RK4 with step size 0.05
	
\end{itemize}

\item\textbf{Damped Pendulum}
\begin{itemize}[noitemsep,nolistsep]

	\item Dimension of state ($m$): 6
	\item Dimension of auxiliary variables ($h$): 2
	\item Dimension of observations ($d$): 2
	\item Number of base matrices ($K$): 5

	\item Integrator: RK4 with step size 0.05
	
\end{itemize}

\item\textbf{CMU Motion Capture (Walking)}
\begin{itemize}[noitemsep,nolistsep]

	\item Dimension of state ($m$): 10
	\item Dimension of auxiliary variables ($h$): 6
	\item Dimension of observations ($d$): 50

	\item Integrator: RK4 with step size 0.05
	
\end{itemize}

\item\textbf{USHCN Climate Indicators}
\begin{itemize}[noitemsep,nolistsep]

	\item Dimension of state ($m$): 10
	\item Dimension of auxiliary variables ($h$): 5
	\item Dimension of observations ($d$): 5
	\item Number of base matrices ($K$): 10

	\item Integrator: Euler with step size 0.1
	
\end{itemize}

\item\textbf{Pymunk Physical Environments}
\begin{itemize}[noitemsep,nolistsep]

	\item Dimension of state ($m$): 10
	\item Dimension of auxiliary variables ($h$): 4
	\item Dimension of observations ($d$): 1024
	\item Number of base matrices ($K$): 10

	\item Integrator: RK4 with step size 0.05
	
\end{itemize}
\end{itemize}

\section{Additional Results}
\label{app:additional-results}
Table~\ref{tab:parameter-comparison} shows the number of trainable parameters in each model for different experiments. \ourmodel\ models obtain better performance on every dataset with significantly fewer parameters. Table~\ref{tab:ols-pendulum} shows the goodness-of-fit coefficient ($R^2$) for ordinary least squares regression with the latent states as features, and the ground truth angle and angular velocity as targets. \ctssmnl\ and \ctssmll\ models obtain a high $R^2$ coefficient showing that the latent states learned by these models are informative about the true latent state (angle and angular velocity). 

Figs.~\ref{fig:bb-all-preds} and \ref{fig:pendulum-all-preds} show sample predictions from the \emph{best run} of each model for different missing data settings on the bouncing ball and the damped pendulum datasets, respectively. For the bouncing ball experiment, both LatentODE and LatentSDE learn that the dataset exhibits a zig-zag pattern but are unable to accurately extrapolate it beyond the training context. In the case of damped pendulum, LatentODE and LatentSDE perform well on the low missing data settings (0\% and 30\%) but completely fail on the more challenging settings of 50\% and 80\% missing data. In contrast, \ctssmnl\ and \ctssmll\ generate accurate predictions across datasets and missing data settings. Furthermore, while the predictions shown in Figs.~\ref{fig:bb-all-preds} and \ref{fig:pendulum-all-preds} are from the best performing runs of each model, they represent a typical run for \ctssmnl\ and \ctssmll. On the other hand, the prediction quality from LatentODE and LatentSDE models varies significantly across random initializations.

Fig.~\ref{fig:pymunk-emd} shows the variation of the EMD with time for different models on the box and pong datasets. All models have EMD close to 0 in the context window from 0-2s; however, in the forecast horizon from 2-6s, the EMD rises gradually for \ctssmnl\ and \ctssmll\ but rapidly and irregularly for other models. Figs.~\ref{fig:box-all-preds} and \ref{fig:pong-all-preds} show sample predictions from different models on the box and the pong datasets, respectively. \ctssmnl\ and \ctssmll\ generate accurate predictions whereas LatentODE and LatentSDE perform significantly worse.
\begin{table}[htb]
    \footnotesize
	\centering
    \vspace{-1em}
	\caption{The number of trainable parameters in every model for different experiments.}
    \label{tab:parameter-comparison}
    \begin{tabular}{lrrrrr}
		\toprule
        \multirow{2}{*}{Model} & \multicolumn{5}{c}{Number of Parameters}\\
        \cmidrule{2-6}
         & Bouncing Ball & Damped Pendulum & MoCap Walking (Setup 2) & USHCN & Pymunk Environments\\
        \midrule
        LatentODE & 2094 & 3336 & 15454 & -- & 204243 \\
        LatentSDE & 5461 & 5557 & 17187 & -- & 208043 \\
        GRUODE-B & 32207 & 39884 & -- & -- & -- \\
        \midrule
        \ctssmlti\  & 63 & 72 & 11080 & 185 & 165911 \\
        \ctssmnl\  & 859 & 862 & 11620 & 1439 & 167165 \\
        \ctssmll\  & 974 & 977 & 12509 & 2439 & 168165 \\
        \bottomrule
    \end{tabular}
\end{table}

\begin{table}[htb]
	\footnotesize
	\centering
    \vspace{-1em}
	\caption{Goodness-of-fit coefficient ($R^2$) of ordinary least squares (OLS) regression for the \emph{best run} of each model on the Pendulum dataset. The latent states are treated as features and ground truth angle --- transformed into polar coordinates: $\sin(\text{angle})/\cos(\text{angle})$ --- and angular velocity as targets.}
    \label{tab:ols-pendulum}
	\begin{tabular}{lcccccccc}
		\toprule
        \multirow{2}{*}{Model}   &     \multicolumn{4}{c}{$\sin(\text{angle})/\cos(\text{angle})$ $R^2$  ($\uparrow$) (\% Missing)}  & \multicolumn{4}{c}{Angular Velocity $R^2$ ($\uparrow$) (\% Missing)}\\
              \cmidrule(lr){2-5} \cmidrule(lr){6-9}
		& 0\%  &   30\%  &      50\%  & 80\% & 0\%  &   30\%  &      50\%  & 80\% \\
        \midrule
        LatentODE & 0.000 / 0.802     & 0.000 / 0.735     & 0.000 / 0.744     & 0.000 / 0.626     & 0.001 & 0.000 & 0.000 & 0.000 \\
        LatentSDE & 0.953 / 0.960  & 0.918 / 0.957 & 0.000 / 0.817     & 0.000 / 0.513     & 0.970 & 0.962 & 0.001 & 0.000 \\
        \midrule
        \ctssmlti\ & 0.593 / 0.537 & 0.604 / 0.468 & 0.477 / 0.796 & 0.481 / 0.705 & 0.349 & 0.388 & 0.162 & 0.305 \\
        \ctssmnl\ & 0.984 / \textbf{0.990}  & 0.982 / 0.985 & \textbf{0.973} / 0.976 & \textbf{0.905} / \textbf{0.920}  &\textbf{ 0.986} & 0.969 & 0.935 & \textbf{0.859} \\
        \ctssmll\ & \textbf{0.986} / 0.989 & \textbf{0.983} / \textbf{0.989} & 0.972 /\textbf{ 0.980}  & 0.875 / 0.888 & 0.972 & \textbf{0.978} & \textbf{0.955} & 0.827\\
       
		\bottomrule
	\end{tabular}
\end{table}

\begin{figure}[ht]
    \centering
    \subfloat[0\% Missing]{\includegraphics[width=0.45\linewidth]{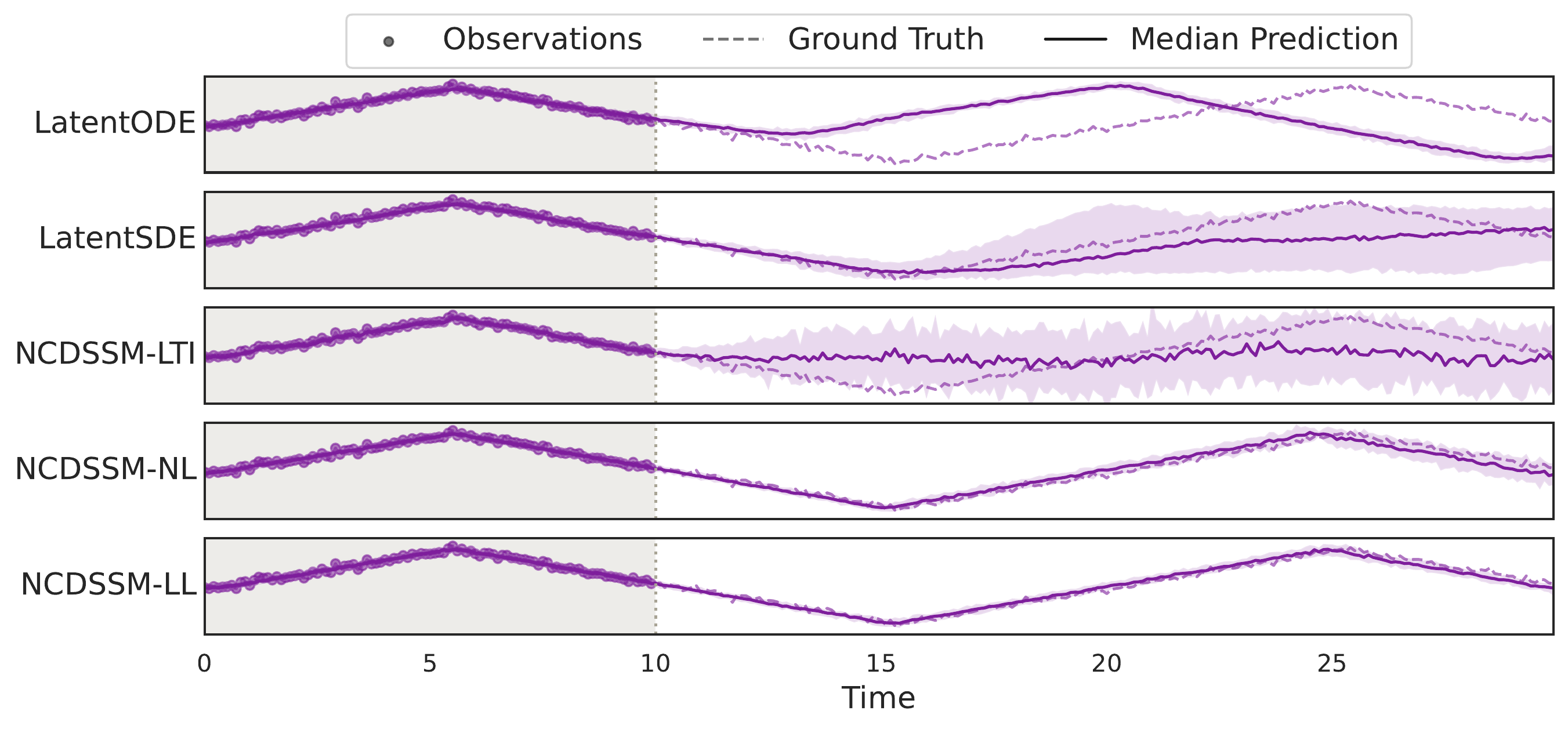}}
    \subfloat[30\% Missing]{\includegraphics[width=0.45\linewidth]{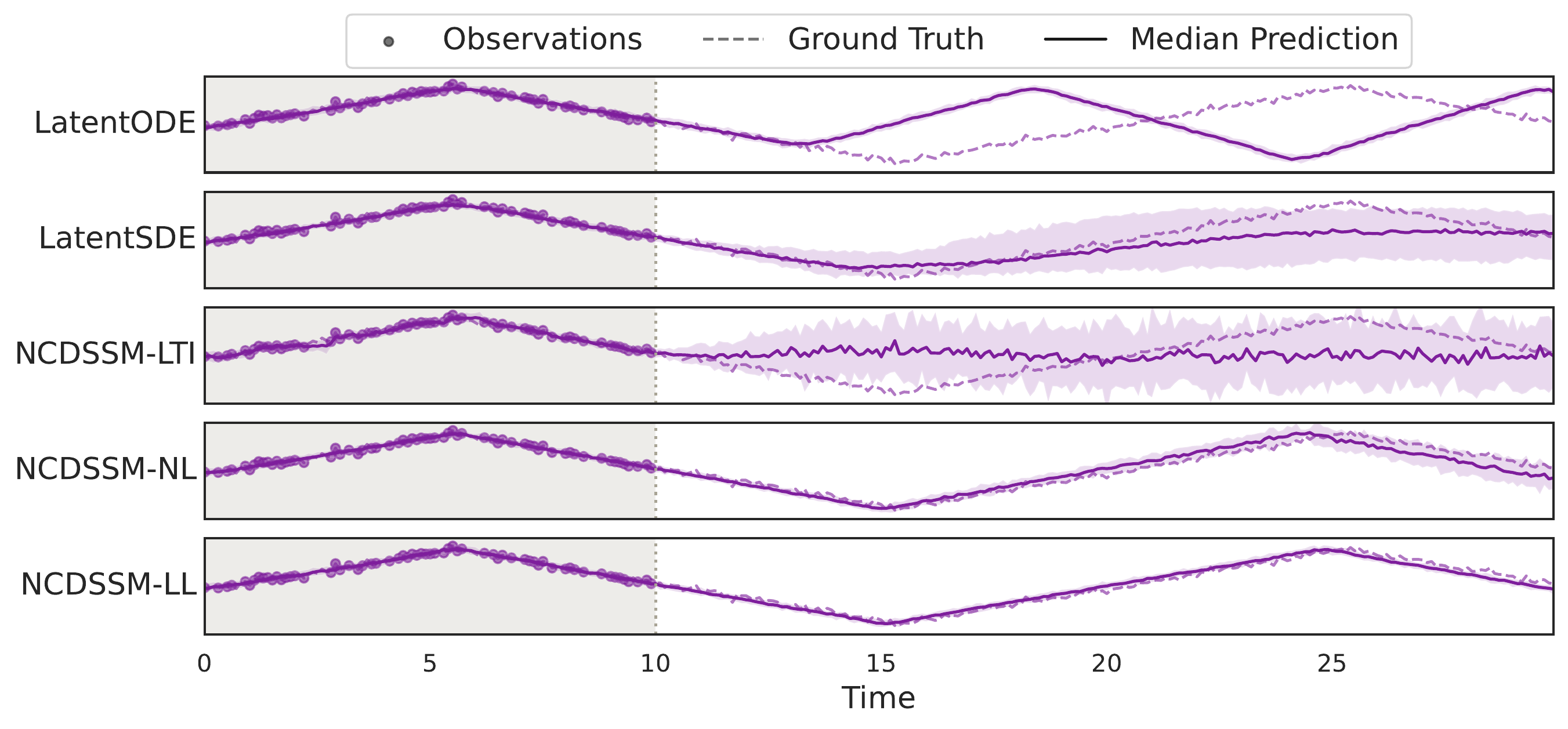}}\\
    \subfloat[50\% Missing]{\includegraphics[width=0.45\linewidth]{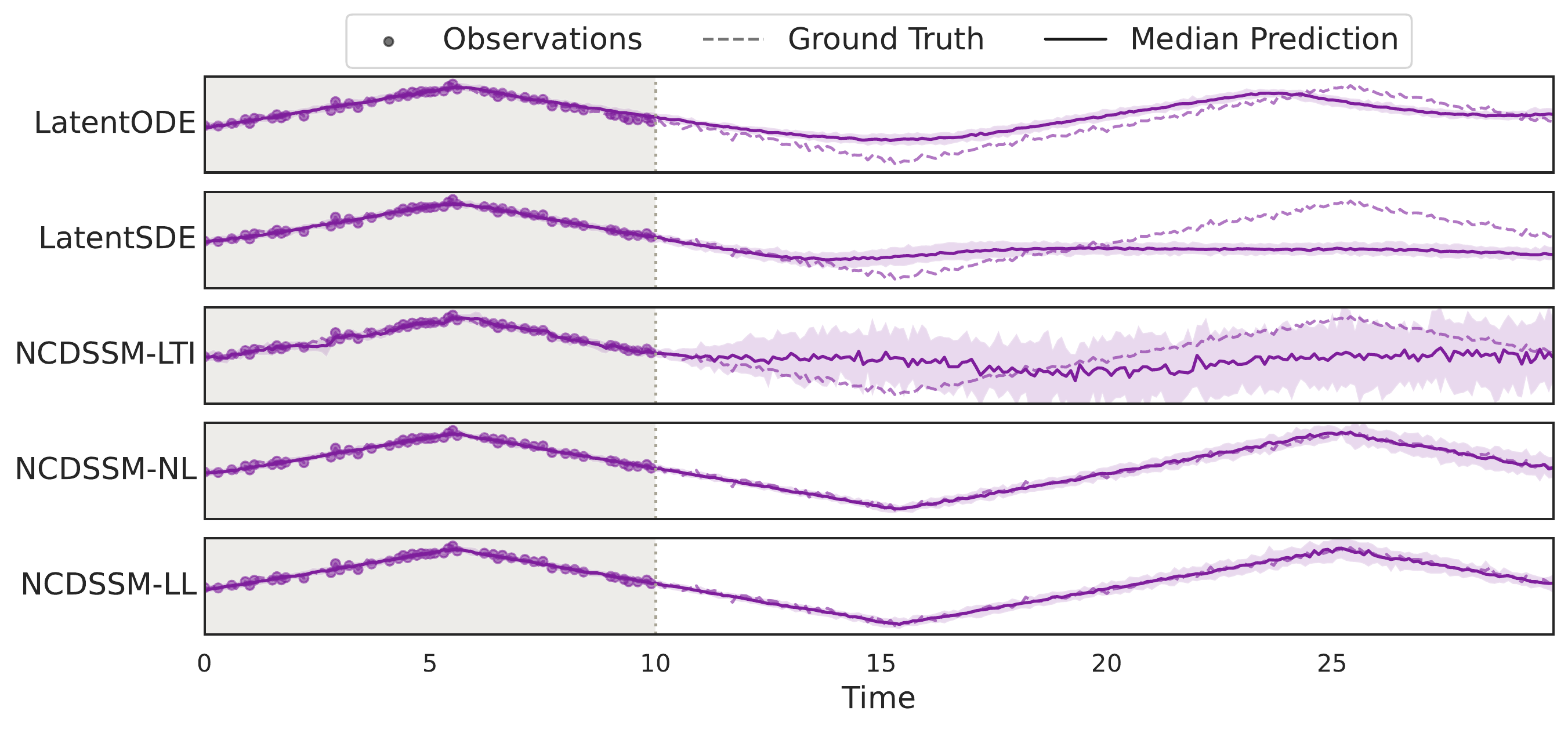}}
    \subfloat[80\% Missing]{\includegraphics[width=0.45\linewidth]{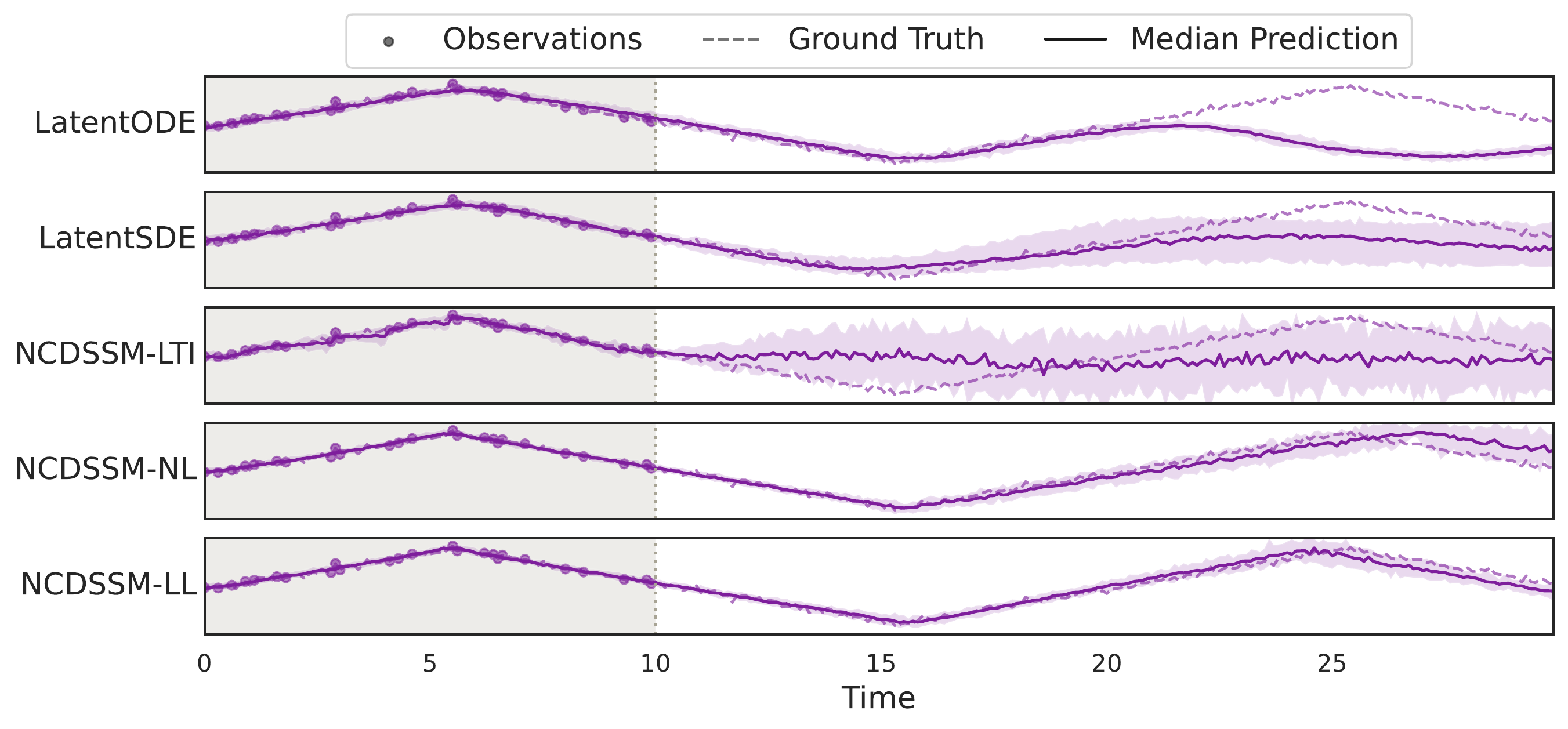}}
    \caption{Predictions from different models on the bouncing ball dataset for the 0\%, 30\%, 50\%, and 80\% missing data settings. The ground truth is shown using dashed lines with observed points in the context window (gray shaded region) shown as filled circles. The vertical dashed gray line marks the beginning of the forecast horizon. Solid lines indicate median predictions with 90\% prediction intervals shaded around them.}
    \label{fig:bb-all-preds}
\end{figure}

\begin{figure}[hb]
    \centering
    \subfloat[0\% Missing]{\includegraphics[width=0.45\linewidth]{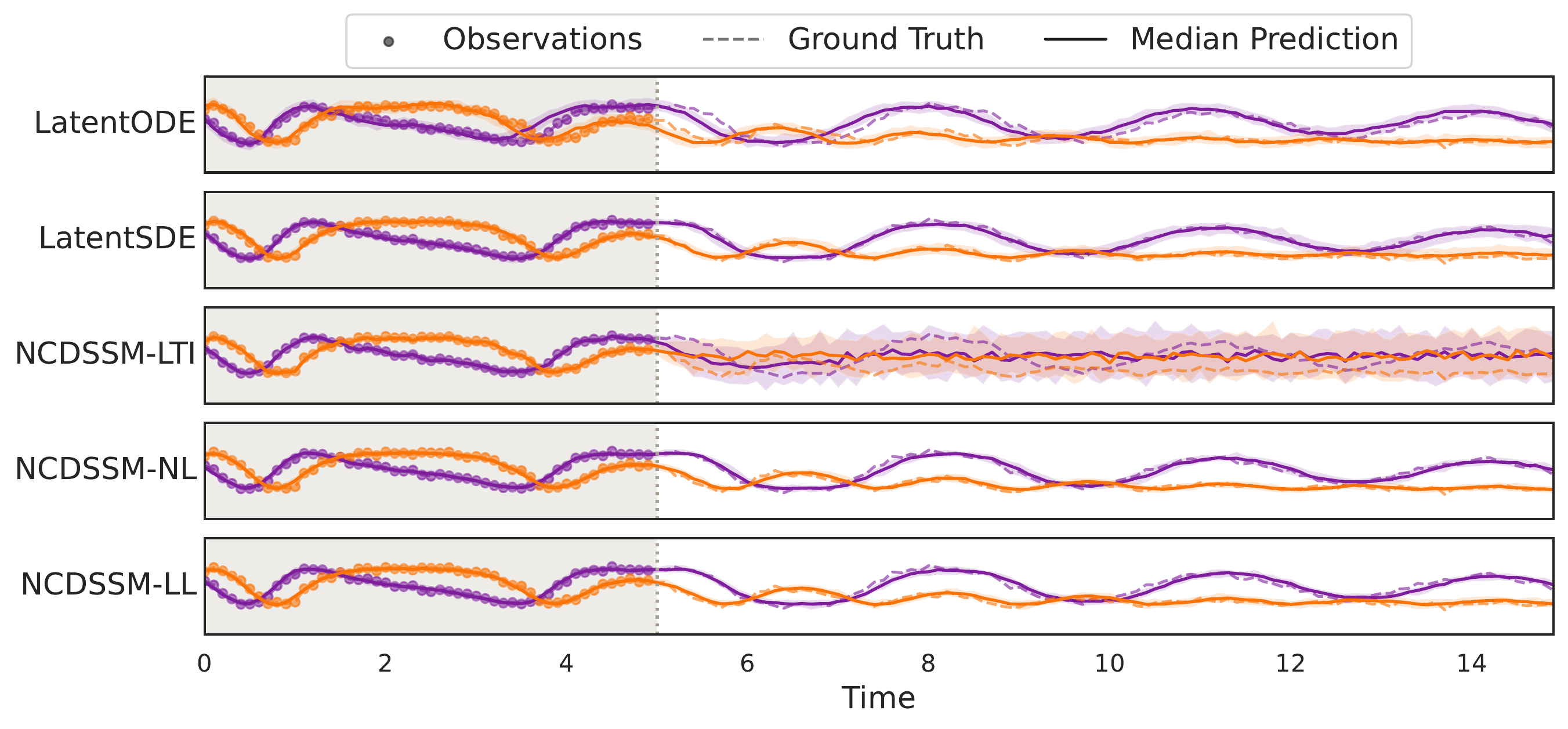}}
    \subfloat[30\% Missing]{\includegraphics[width=0.45\linewidth]{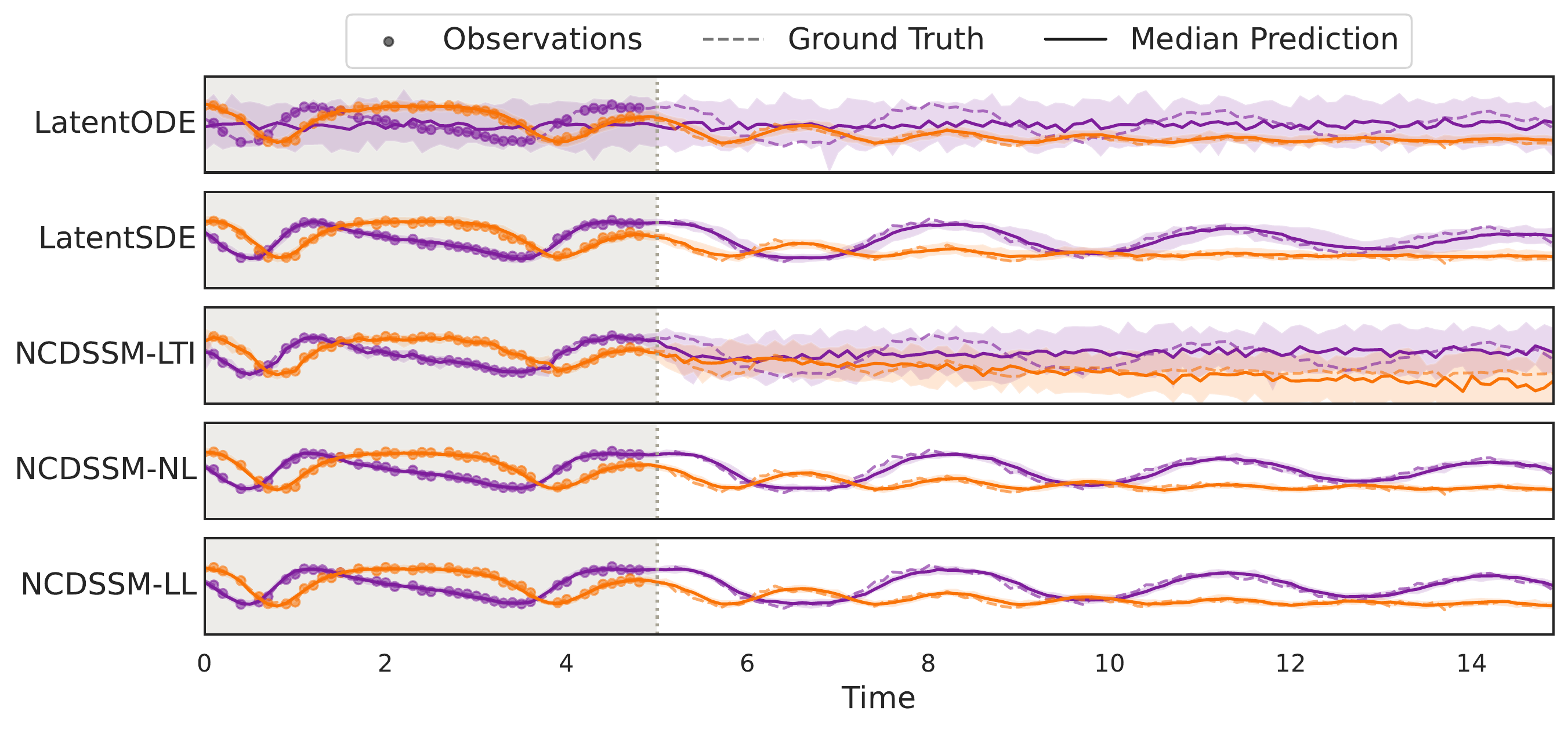}}\\
    \subfloat[50\% Missing]{\includegraphics[width=0.45\linewidth]{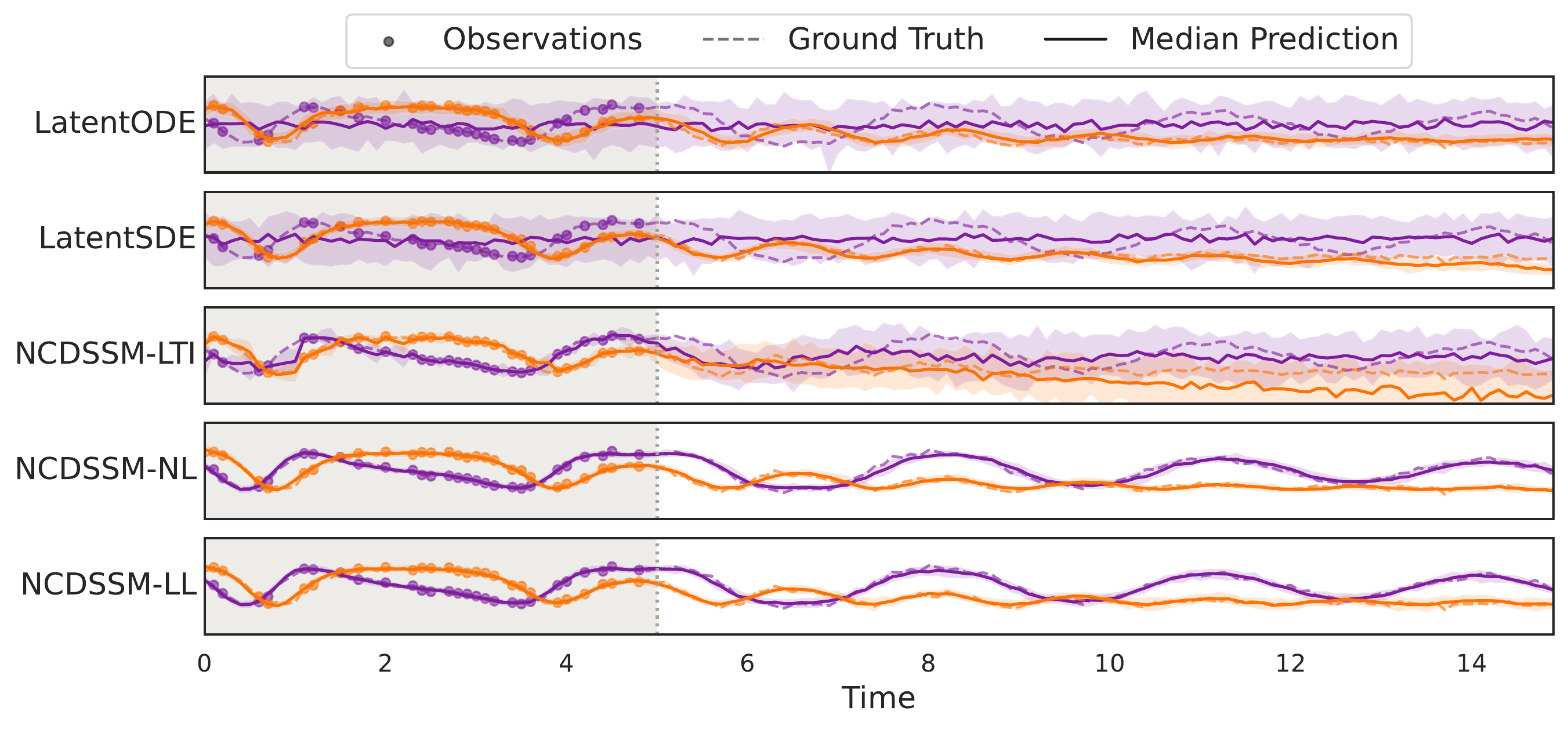}}
    \subfloat[80\% Missing]{\includegraphics[width=0.45\linewidth]{pendulum_all_0.8.pdf}}
    \caption{Predictions from different models on the damped pendulum dataset for the 0\%, 30\%, 50\%, and 80\% missing data settings. The ground truth is shown using dashed lines with observed points in the context window (gray shaded region) shown as filled circles. The vertical dashed gray line marks the beginning of the forecast horizon. Solid lines indicate median predictions with 90\% prediction intervals shaded around them. The purple and orange colors indicate observation dimensions.}
    \label{fig:pendulum-all-preds}
\end{figure}

\begin{figure}[ht]
    \centering
    \subfloat{\includegraphics[width=0.5\linewidth]{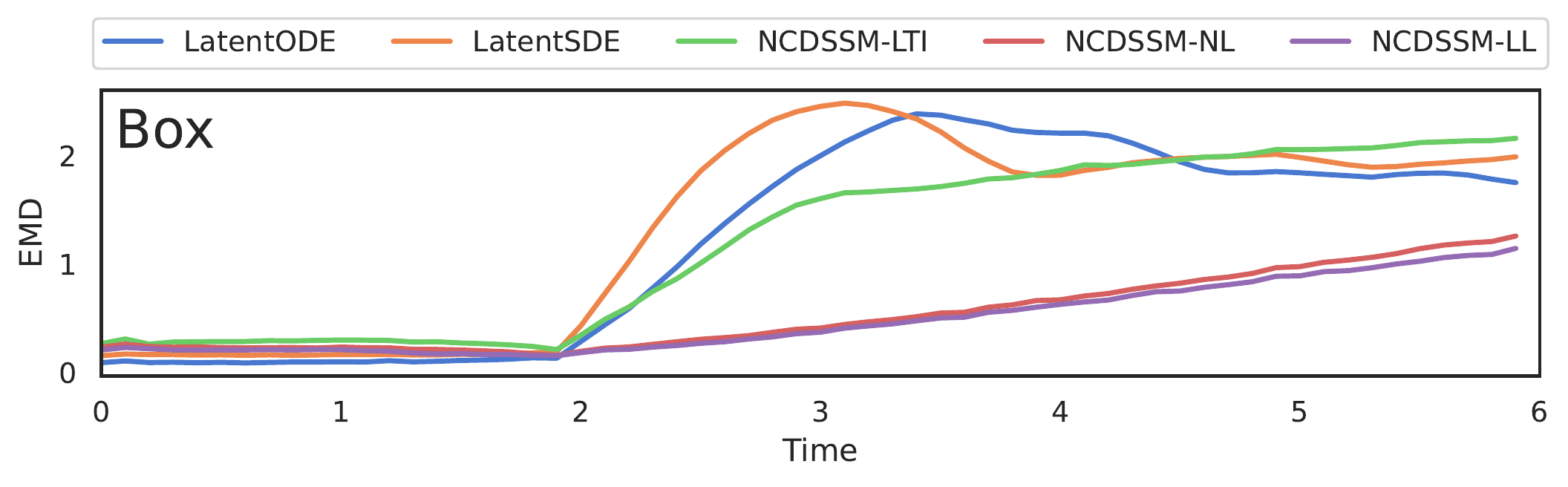}}
    \subfloat{\includegraphics[width=0.5\linewidth]{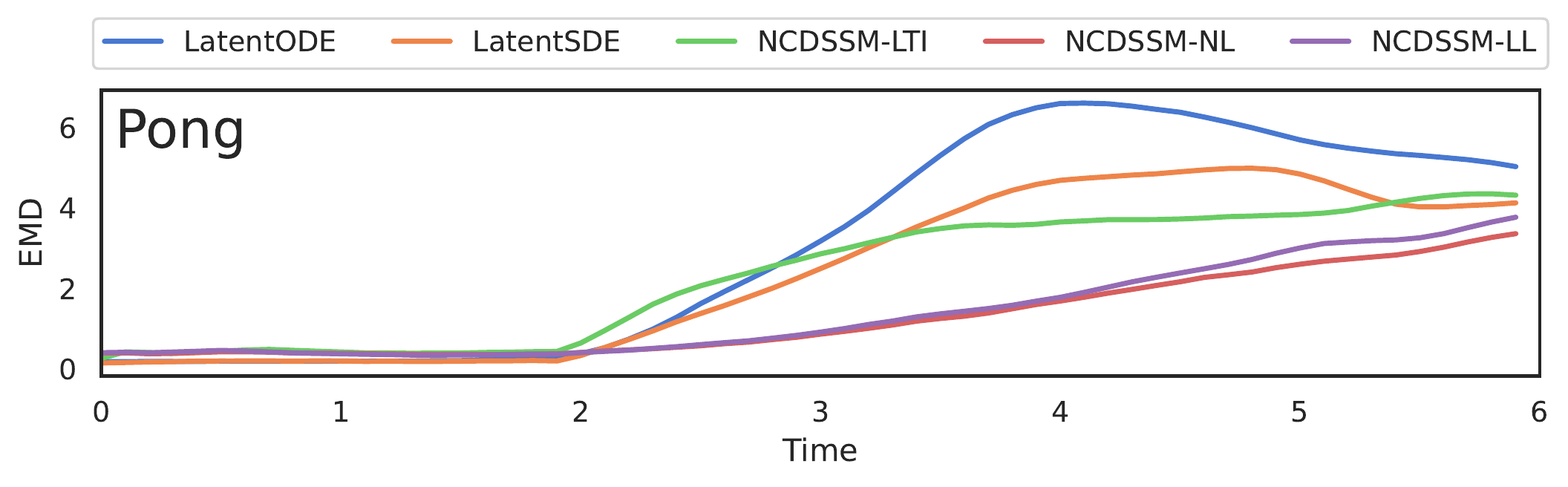}}
    \caption{Variation of EMD over time for the Box (left) and Pong (right) datasets. The EMD rises gradually with time for \ctssmll\ and \ctssmnl\ but rapidly and irregularly for other models.}
    \label{fig:pymunk-emd}
\end{figure}

\begin{figure*}[hb]
    \centering
    \subfloat[Box LatentODE]{\includegraphics[width=0.9\linewidth]{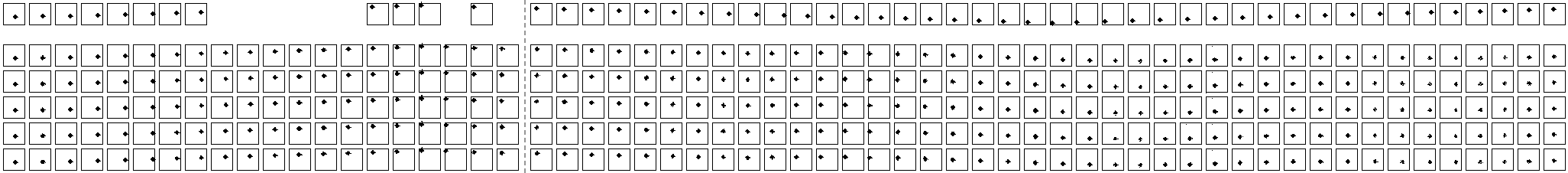}}\\
    \subfloat[Box LatentSDE]{\includegraphics[width=0.9\linewidth]{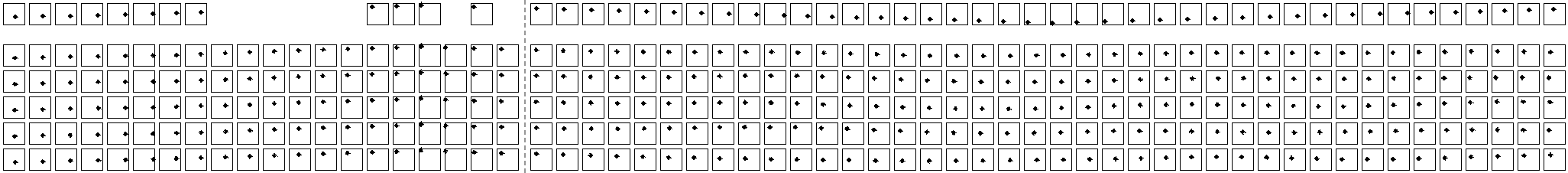}}\\
    \subfloat[Box \ctssmlti]{\includegraphics[width=0.9\linewidth]{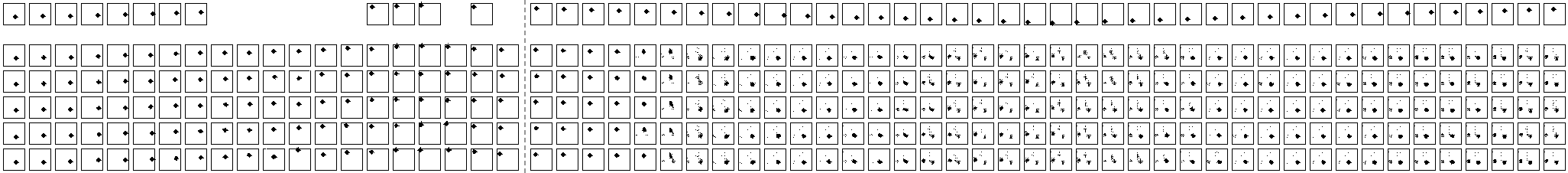}}\\
    \subfloat[Box \ctssmnl]{\includegraphics[width=0.9\linewidth]{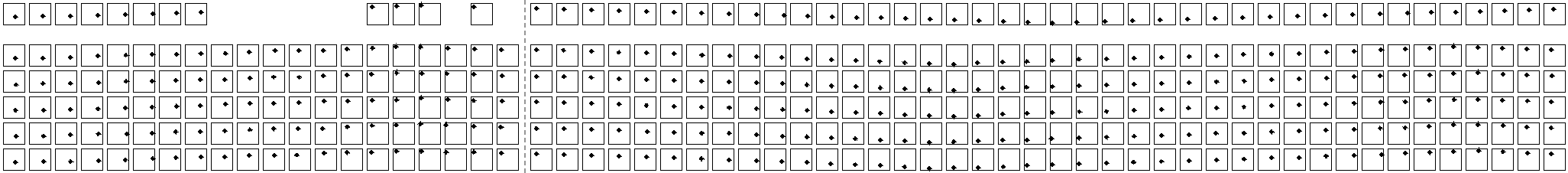}}\\
    \subfloat[Box \ctssmll]{\includegraphics[width=0.9\linewidth]{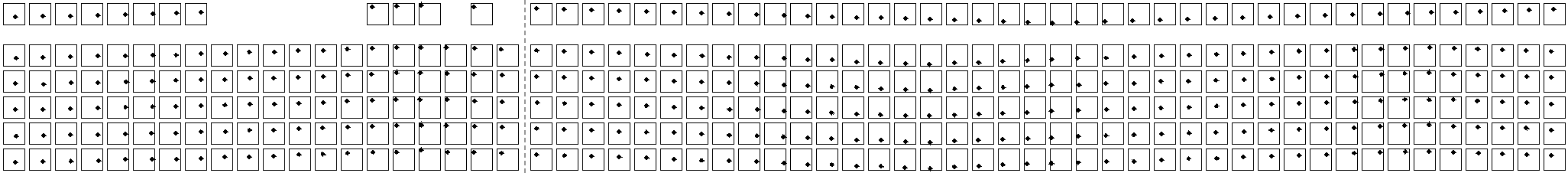}}
    \caption{Sample predictions from different models on the Box dataset. The top row in each figure is the ground truth with some missing observations in the context window (before the dashed grey line). The next five rows show trajectories sampled from each model. Best viewed zoomed-in on a computer.}
    \label{fig:box-all-preds}
\end{figure*}

\begin{figure*}[htb]
    \centering
    \subfloat[Pong LatentODE]{\includegraphics[width=0.9\linewidth]{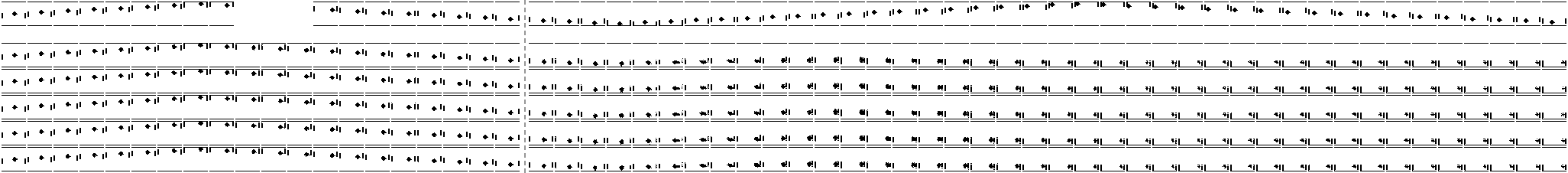}}\\
    \subfloat[Pong LatentSDE]{\includegraphics[width=0.9\linewidth]{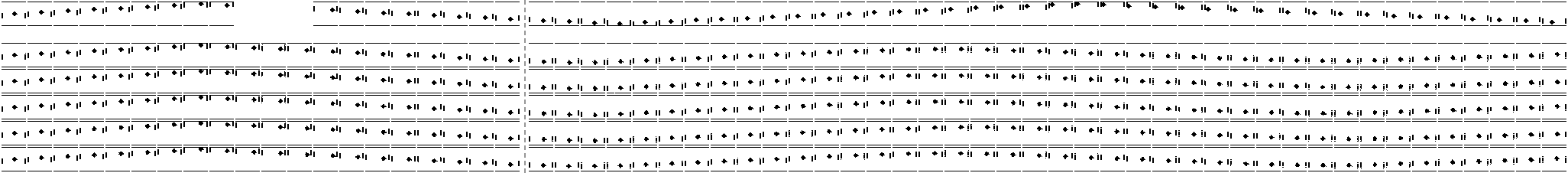}}\\
    \subfloat[Pong \ctssmlti]{\includegraphics[width=0.9\linewidth]{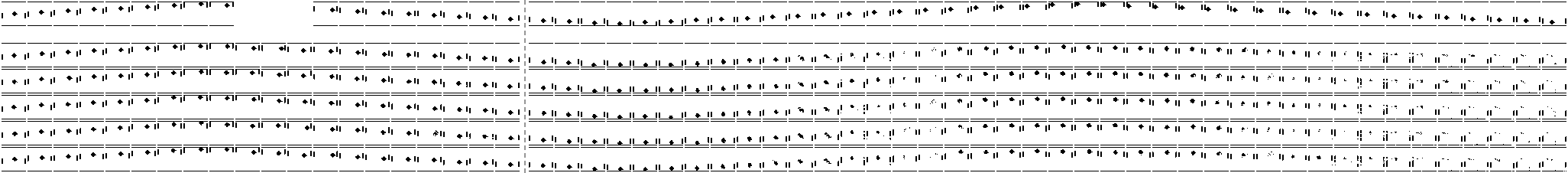}}\\
    \subfloat[Pong \ctssmnl]{\includegraphics[width=0.9\linewidth]{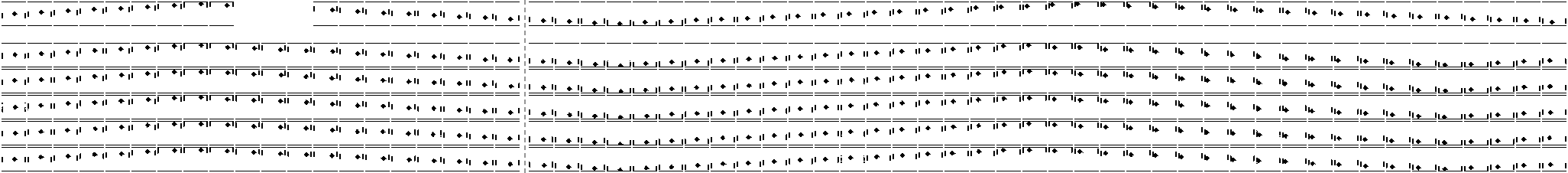}}\\
    \subfloat[Pong \ctssmll]{\includegraphics[width=0.9\linewidth]{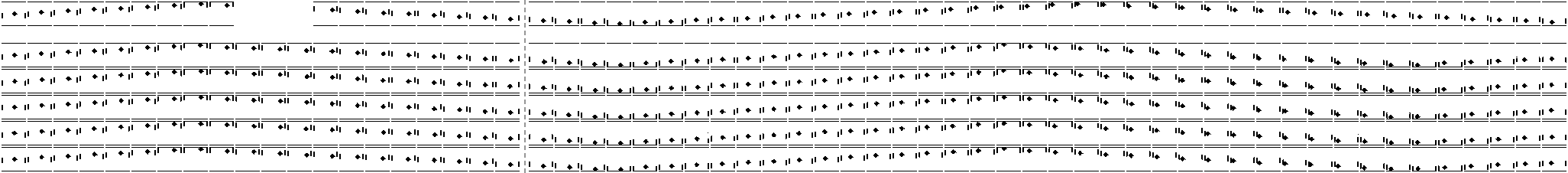}}
    \caption{Sample predictions from different models on the Pong dataset. The top row in each figure is the ground truth with some missing observations in the context window (before the dashed grey line). The next five rows show trajectories sampled from each model. Best viewed zoomed-in on a computer.}
    \label{fig:pong-all-preds}
\end{figure*}

\end{document}